\documentclass[11pt,a4paper]{article}

\usepackage[margin=1in]{geometry}
\usepackage{amsmath,amssymb,amsfonts,amsthm,mathrsfs}
\usepackage[ruled,vlined,linesnumbered]{algorithm2e}
\usepackage{graphicx,multirow,booktabs}
\usepackage{xcolor,textcomp}
\usepackage{tabularx,makecell}
\usepackage[title]{appendix}
\usepackage{hyperref}
\usepackage[numbers,sort&compress]{natbib}
\usepackage{caption}
\usepackage{subcaption}

\usepackage{authblk}

\newtheorem{theorem}{Theorem}

\newtheorem{proposition}[theorem]{Proposition}
\newtheorem{corollary}[theorem]{Corollary}
\theoremstyle{definition}
\newtheorem{definition}[theorem]{Definition}
\theoremstyle{remark}
\newtheorem{remark}[theorem]{Remark}

\title{Adversarial Vulnerabilities in Neural Operator Digital Twins:\\ Gradient-Free Attacks on Nuclear Thermal-Hydraulic Surrogates}

\author[1]{Samrendra Roy\thanks{Corresponding author: \texttt{roysam@illinois.edu}}}
\author[1]{Kazuma Kobayashi}
\author[2]{Souvik Chakraborty}
\author[1]{Rizwan-uddin}
\author[1,3]{Syed Bahauddin Alam}
\affil[1]{Department of Nuclear, Plasma \& Radiological Engineering, University of Illinois Urbana-Champaign, Urbana, IL 61801, USA}
\affil[2]{Department of Applied Mechanics, Indian Institute of Technology Delhi, New Delhi 110016, India}
\affil[3]{National Center for Supercomputing Applications, University of Illinois Urbana-Champaign, Urbana, IL 61801, USA}

\date{}

\begin{document}

\maketitle
\begin{abstract}
Operator learning models are rapidly emerging as the predictive core of digital twins for nuclear and energy systems, promising real-time field reconstruction from sparse sensor measurements. Yet their robustness to adversarial perturbations remains uncharacterized, a critical gap for deployment in safety-critical systems. Here we show that neural operators are acutely vulnerable to extremely sparse (fewer than 1\% of inputs), physically plausible perturbations that exploit their sensitivity to boundary conditions. Using gradient-free differential evolution across four operator architectures, we demonstrate that minimal modifications trigger catastrophic prediction failures, increasing relative $L_2$ error from $\sim$1.5\% (validated accuracy) to 37--63\% while remaining completely undetectable by standard validation metrics. Notably, 100\% of successful single-point attacks pass z-score anomaly detection. We introduce the \emph{effective perturbation dimension} $d_{\text{eff}}$, a Jacobian-based diagnostic that, together with sensitivity magnitude, yields a two-factor vulnerability model explaining why architectures with extreme sensitivity concentration (POD-DeepONet, $d_{\text{eff}} \approx 1$) are not necessarily the most exploitable, since low-rank output projections cap maximum error, while moderate concentration with sufficient amplification (S-DeepONet, $d_{\text{eff}} \approx 4$) produces the highest attack success. Gradient-free search outperforms gradient-based alternatives (PGD) on architectures with gradient pathologies, while random perturbations of equal magnitude achieve near-zero success rates, confirming that the discovered vulnerabilities are structural. Our findings expose a previously overlooked attack surface in operator learning models and establish that these models require robustness guarantees beyond standard validation before deployment.
\end{abstract}

\vspace{1em}

\section{Introduction}\label{sec:intro}

Scientific machine learning has been transformed by the emergence of neural operators, architectures that learn mappings between infinite-dimensional function spaces rather than finite-dimensional input-output pairs~\cite{kovachki2023neural}. Beginning with the universal approximation theorem for operators~\cite{chen1995universal}, this field has rapidly matured through foundational contributions including the Deep Operator Network (DeepONet)~\cite{lu2021deeponet}, which decomposes the operator into branch and trunk networks; the Fourier Neural Operator (FNO)~\cite{li2020fourier}, which achieves mesh-independent learning through spectral convolutions; and more recent extensions such as physics-informed DeepONets~\cite{wang2021learning}, multifidelity operator networks~\cite{lu2022multifidelity}, and geometry-adaptive formulations~\cite{li2023fourier}. These models achieve order-of-magnitude computational speedups (from hours of high-fidelity simulation to millisecond inference) while maintaining relative errors below 1--2\% across diverse partial differential equation (PDE) families~\cite{lu2022comprehensive}. Their success has catalyzed a broader scientific ML toolkit encompassing graph neural operators~\cite{li2020multipole}, transformer-based architectures~\cite{cao2021choose}, wavelet-enhanced operators~\cite{tripura2023wavelet}, and state-space neural operators~\cite{hu2024state}, collectively establishing operator learning as a foundational tool for computational science and engineering.

This computational efficiency has made neural operators the predictive backbone of digital twins for safety-critical infrastructure, where real-time field reconstruction from sparse sensor measurements is essential for operational decision-making. In nuclear energy systems, DeepONet-based surrogate models have been deployed for real-time thermal-hydraulic monitoring~\cite{kobayashi2024deeponet, kobayashi2024improved}, virtual sensor networks that predict field quantities in inaccessible reactor regions~\cite{hossain2025virtual}, rapid safety margin evaluation during transient scenarios~\cite{kropaczek2023digital}, and hybrid data-driven frameworks for reactor power prediction~\cite{daniell2025digitaltwin}. Beyond nuclear applications, operator surrogates now underpin subsurface flow modeling for geological carbon sequestration~\cite{jiang2024fourier} and weather and climate prediction at a global scale~\cite{pathak2022fourcastnet}. The urgency of deployment is accelerating: recent U.S.\ federal mandates prioritize AI-enabled nuclear infrastructure as a national security imperative, with deployment timelines as short as 30 months~\cite{whitehouse2025nuclear}, while the International Atomic Energy Agency has initiated coordinated research projects specifically on AI and digital twin integration for reactor operations, including emerging efforts toward domain-specific foundation models for reactor control~\cite{lee2025agentic}. This deployment trajectory demands rigorous assessment of failure modes that extend beyond standard accuracy validation.

In conventional deep learning, adversarial robustness has been rigorously characterized over the past decade. Since Szegedy et al.~\cite{szegedy2013intriguing} first demonstrated that imperceptible perturbations can cause confident misclassifications, a broad taxonomy of attack methods has emerged: gradient-based approaches such as FGSM~\cite{goodfellow2014explaining} and PGD~\cite{madry2018towards}; optimization-based attacks like C\&W~\cite{carlini2017towards}; and black-box methods including transfer attacks~\cite{papernot2017practical}, score-based estimation~\cite{chen2017zoo}, and evolutionary strategies~\cite{su2019onepixel}. Correspondingly, defense mechanisms have been extensively studied, from adversarial training~\cite{madry2018towards} and certified robustness via randomized smoothing~\cite{cohen2019certified} to input purification~\cite{shi2021online}, defensive distillation~\cite{papernot2016distillation}, and Jacobian regularization~\cite{hoffman2019robust}. A critical insight from this literature is that robustness is at odds with standard accuracy in classification settings, a phenomenon formalized through the accuracy-robustness trade-off~\cite{tsipras2019robustness, zhang2019theoretically}. However, this entire body of work operates within the classification or bounded regression framework, where perturbation budgets are defined in pixel space ($\ell_p$ norms), success is measured by label flipping, and the physical meaning of perturbations is rarely considered. The extension to function-space mappings with physical constraints remains largely unexplored.

The challenge of adversarial robustness in neural operator-based scientific ML differs sharply from established attack methods in several critical respects. First, neural operators produce high-dimensional continuous field outputs ($\mathbb{R}^m$ with $m \sim 10^4$--$10^6$) rather than discrete class labels, meaning adversarial success must be quantified through field-level error metrics rather than misclassification rates. Second, the input space carries physical semantics: boundary conditions, initial conditions, material properties, or sensor measurements, all of which must respect governing equations, conservation laws, thermodynamic bounds, and sensor calibration ranges. Perturbations that violate these constraints are trivially detectable and operationally irrelevant. Third, the consequence of adversarial failure is not a misclassified image but a physically plausible yet grossly wrong field prediction: a reactor digital twin reporting safe thermal margins while actual temperatures approach material failure limits, or a carbon sequestration model predicting stable CO$_2$ plume migration while breakthrough is imminent. These distinctions demand a different threat model: one where the adversary must craft inputs that simultaneously maximize prediction error, satisfy all physical feasibility constraints, and remain indistinguishable from nominal operating variability through standard validation metrics.

Prior work connecting adversarial methods with physics-informed machine learning has focused predominantly on using adversarial training as a \emph{cooperative} strategy for improving PINN accuracy, rather than characterizing \emph{malicious} vulnerability. Shi et al.~\cite{shi2023wbar} proposed white-box adversarial sampling for PINNs to improve convergence in regions of high PDE residual, while competitive physics-informed training~\cite{zeng2023competitive} uses adversarial dynamics between collocation point selection and network training to reduce approximation error. Shekarpaz et al.~\cite{shekarpaz2022piat} introduced physics-informed adversarial training that explicitly frames the residual minimization as a minimax game. The MetaPhysiCa framework~\cite{mouli2024metaphysica} addressed out-of-distribution robustness in physics-informed models through metalearning. Structure-preserving approaches~\cite{chu2024structure} have shown that embedding Lyapunov stability into neural ODE architectures improves robustness to white-box attacks. While these contributions establish that physics constraints can enhance robustness, they address a complementary problem: \emph{improving model training through adversarial dynamics} rather than \emph{evaluating vulnerability to deliberate input manipulation}. More broadly, the adversarial ML literature has concentrated almost exclusively on classification tasks; adversarial attacks on continuous-output regression surrogates, where success is measured by field-level error rather than label flipping and perturbations must satisfy physical constraints, remain largely unexplored. No prior study has systematically characterized how trained neural operator surrogates, i.e.\ models deployed for inference rather than being actively trained, respond to adversarial perturbations designed to induce operational failures. A recent analysis of fundamental limitations in physics-informed neural networks~\cite{naser2025flaws} demonstrates that such architectures can harbor systematic risks that produce false confidence through predictions appearing physically plausible while masking critical omissions, further motivating explicit adversarial evaluation.

The cyber-physical security context provides concrete motivation for this adversarial analysis. Industrial control systems (ICS) and SCADA networks governing critical infrastructure face a persistent and escalating threat landscape~\cite{rahman2026securing}, as illustrated by the Stuxnet precedent demonstrating that cyber weapons targeting programmable logic controllers can cause physical destruction to nuclear infrastructure. More broadly, intrusions against SCADA and human-machine interface systems across energy, water, and transportation sectors increasingly target the data pipeline between physical sensors and digital monitoring systems. In this threat environment, AI-based surrogate models introduce a novel attack surface: rather than manipulating actuators or control signals directly, an adversary can target the sensor-to-prediction pathway by corrupting the inputs to a neural operator digital twin. Our work does not demonstrate a full-stack cyber-physical attack; rather, we characterize the vulnerability of the ML inference layer under the assumption that an adversary has achieved the ability to corrupt sensor readings before they reach the surrogate model, a capability consistent with known SCADA intrusion vectors. If the corrupted inputs remain within physical bounds and the resulting predictions appear reasonable by aggregate metrics, such attacks would be effectively invisible to existing anomaly detection systems. This threat model is especially acute for neural operator surrogates because their learned representations compress high-dimensional physics into low-dimensional latent spaces that may be highly sensitive to particular input directions, a property we formalize and quantify in this work.

Here, we show that neural operator surrogates trained on high-fidelity simulation data are acutely vulnerable to physics-compliant adversarial perturbations of extreme sparsity: modifications affecting fewer than 1\% of inputs that degrade validated models from production-ready accuracy ($\sim$1.5\% relative $L_2$ error) to 37--63\% error depending on architecture. In applications where surrogate predictions directly inform safety margins, thermal limits, or control actions, errors of this magnitude, representing an order-of-magnitude collapse from validated performance, can render predictions operationally unreliable. This degradation occurs despite perturbations satisfying all physical constraints: modified sensor readings fall within calibration ranges, boundary conditions respect conservation laws, and input distributions remain consistent with normal operating variability (Section~\ref{sec:feasibility}). The core failure mechanism is a \emph{sensitivity mismatch} in learned operator mappings that, to the best of our knowledge, has not been characterized in the neural operator context: single-coordinate perturbations can trigger spatially nonlocal field errors through basis function coupling (in spectral architectures) or latent space propagation (in decomposition-based models), producing hotspot relocations, flow pattern inversions, and safety margin violations that aggregate error metrics fail to capture. We formalize this vulnerability through the \emph{effective perturbation dimension} $d_{\text{eff}}$, a Jacobian-based metric quantifying the concentration of input-output sensitivity, and demonstrate that attack success depends on both sensitivity concentration and magnitude, yielding a two-factor vulnerability model with distinct architectural risk profiles.

We demonstrate this vulnerability across four operator families (MIMONet, NOMAD, S-DeepONet, POD-DeepONet) using a gradient-free differential evolution framework adapted from one-pixel attacks~\cite{su2019onepixel} but redesigned for continuous regression tasks with physics constraints (Fig.~\ref{fig:DE_encoding}; Fig.~\ref{fig:supp_de_pipeline}; Appendix~\ref{app:note_onepixel}). Our framework requires only black-box model access and knowledge of sensor types and operating ranges (assumptions satisfied by any operator with access to the digital twin's input interface) and operates within realistic computational budgets ($\sim$2 minutes per attack on GPU). The gradient-free approach is essential because neural operators are frequently deployed as proprietary black boxes in industrial settings where internal model parameters and gradients are inaccessible, and because the high-dimensional output space ($\mathbb{R}^{15{,}908}$) makes gradient computation expensive even when model internals are available (Appendix~\ref{app:note_gradfree}).

Our findings carry implications beyond the specific architectures and application studied here. The sensitivity mismatch we identify, whereby models optimized for reconstruction accuracy develop concentrated sensitivity structures that adversaries exploit, is a natural consequence of learning operator compressions from finite training data and is likely to manifest across neural operators deployed in cyber-physical systems. While we demonstrate this on a thermal-hydraulic benchmark, the mechanism is architecture-level rather than application-specific, suggesting that similar vulnerabilities may arise wherever neural operators are used for safety-critical inference. As operator learning models are integrated into increasingly consequential decision pipelines for nuclear monitoring, grid management, and climate prediction, the gap between interpolation accuracy and adversarial robustness that we expose represents a pressing challenge for the trustworthy deployment of scientific machine learning.

\section{Results}\label{sec:results}

\subsection{Cyber-Physical Attack Framework}

Figure~\ref{fig:system_overview} illustrates the complete attack framework for neural operator digital twins. The physical system consists of a heat exchanger with two input branches: Branch 1 encodes global operating parameters (inlet velocity $v_{\text{in}}$, inlet temperature $T_{\text{in}}$), while Branch 2 encodes the spatially varying wall heat flux profile $q''(z)$ along 100 axial locations. Four neural operator architectures learn the mapping from these 102-dimensional boundary conditions to full-field pressure and velocity predictions across 3,977 mesh points.

The attack surface (top panel) demonstrates our differential evolution-based vulnerability discovery process. Given black-box model access and knowledge of sensor types and operating ranges, the attacker uses DE to probe the input space and generate vulnerability heatmaps identifying critical features; indices 28 and 68 of the 100-point $q''(z)$ profile emerge as high-sensitivity locations. All perturbations respect physical bounds: $v_{\text{in}} \in [4.25, 4.77]$ m\,s$^{-1}$, $T_{\text{in}} \in [275.4, 309.4]$ K, which are proper subsets of the design operating envelope (Section~\ref{sec:feasibility}), ensuring attacks remain undetectable by constraint-based anomaly detection.

The impact panel (bottom right) quantifies the consequences: a successful $L_0=3$ attack (modifying just 3 of 102 total inputs) induces 34.2\% relative $L_2$ error in field predictions, causes spatial hotspot displacement, and triggers critical safety metric violations while aggregate validation metrics remain within acceptable tolerances. (This paradox, where models are simultaneously accurate in mean-field error yet severely wrong in safety-critical regions, exposes the fundamental inadequacy of standard validation protocols for adversarial robustness assessment.)

\begin{figure*}[t]
  \centering
  \includegraphics[width=\linewidth]{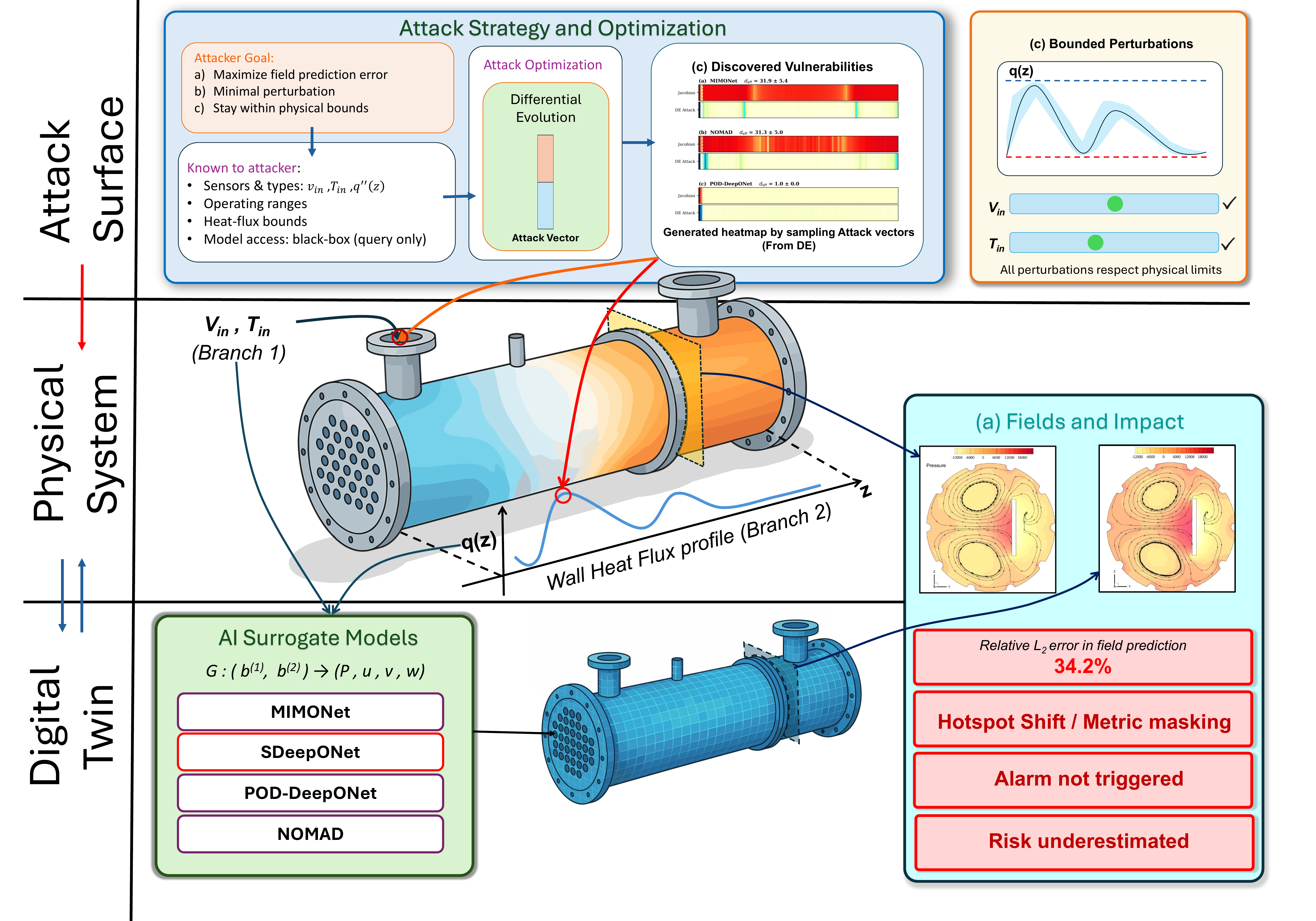}%
\caption{%
\textbf{Differential-evolution-driven vulnerability mapping in a cyber-physical twin.}
\emph{Attack Surface} (top): Using sensor knowledge \((v_{\mathrm{in}},\,T_{\mathrm{in}},\,q''(z))\) and bounds, Differential Evolution probes the input space and yields a heatmap with vulnerability hotspots at indices 28 and 68 of the 100-point \(q''(z)\) profile (Branch 2). All trials respect \(\,v_{\mathrm{in}}\in[4.0,5.0]~\mathrm{m\,s^{-1}},\;T_{\mathrm{in}}\in[263,323]~\mathrm{K}\,\) (panel~c), enabling stealth. \emph{Physical System} (middle): Branch 1 encodes \([v_{\mathrm{in}},T_{\mathrm{in}}]\); Branch 2 encodes axial wall heat flux \(q''(z)\). Single-coordinate injections (\(L_0=1\)) enter via sensor/SCADA/preprocessing tampering. 
\emph{Digital Twin} (bottom): Surrogates (MIMONet, S-DeepONet, POD-DeepONet, NOMAD) map \(G:(b^{(1)},b^{(2)})\!\to\!(P,u,v,w)\); S-DeepONet is most vulnerable. 
\emph{Fields \& impact} (right): Nominal vs attacked temperature shows hotspot displacement and \(34.2\%\) relative \(L_2\) error, achieved with black-box access and physics-respecting inputs, revealing dangerous but bounds-compliant failures in learned operator models.}
  \label{fig:system_overview}
\end{figure*}

\subsection{Benchmark Problem and Baseline Performance}

The benchmark problem represents a single-pass rectangular-channel heat exchanger operating under mixed hydrodynamic and thermal boundary conditions (Fig.~\ref{fig:system_overview}), representative of thermal management designs studied for advanced reactor systems~\cite{ahmed2024enhancing}. Branch 1 encodes inlet velocity $v_{\text{in}} \in [4.0, 5.0]$ m\,s$^{-1}$ and temperature $T_{\text{in}} \in [263, 323]$ K, while Branch 2 encodes the axial wall heat flux distribution $q''(z) = q_{\max} \sin(\pi z/H)$ discretized at 100 points. Each operator maps these 102-dimensional boundary conditions to full-field pressure and velocity predictions across 3,977 mesh points (complete problem specification in Appendix~\ref{app:dataset}).

We compare four neural operator architectures representing distinct design philosophies: \textbf{MIMONet} (1.2M parameters), a direct concatenation-based mapping without explicit operator structure; \textbf{NOMAD} (2.1M parameters), which uses dual MLP encoders with multiplicative branch-trunk fusion; \textbf{S-DeepONet} (1.8M parameters), which employs a recurrent GRU encoder for sequential boundary dependencies with spectral summation composition; and \textbf{POD-DeepONet} (1.6M parameters), which leverages proper orthogonal decomposition to predict coefficients in a low-rank basis rather than full fields (complete architectural specifications in Appendix~\ref{app:architectures}). All four architectures achieve comparable baseline accuracy on clean test data: validation MSE ranges from 0.012 to 0.014, with test set relative $L_2$ errors between 1.3\% and 1.5\% (Table~\ref{tab:supp_baseline}). This near-identical nominal performance establishes that the differential adversarial vulnerability observed in subsequent sections cannot be attributed to accuracy differences; all models meet production-grade standards on traditional validation protocols. The dissociation between interpolation accuracy and adversarial robustness suggests that standard validation metrics based on held-out test error provide insufficient guarantees for deployment in adversarial environments.

\subsection{Differential Vulnerability Across Model Families}
\label{sec:attack_success}

Table~\ref{tab:attack_success} summarizes attack success rates for four neural operator architectures under varying $\ell_0$-norm sparsity budgets ($L_0 = k \in \{1, 3, 5, 10\}$, where $k$ denotes the maximum number of input coordinates perturbed) and error thresholds (10\%, 20\%, 30\%, 40\% relative $L_2$ error). The differential vulnerability observed across model families arises from distinct input-to-field sensitivity structures inherent to each architecture's design.

\begin{table}[t]
\caption{\textbf{Attack success rates across neural operator families.} Percentage of tested samples for which differential evolution found a feasible adversarial example within the specified sparsity budget and error threshold. All perturbations respect physical bounds. The test set comprises 310 samples.}
\label{tab:attack_success}
\centering
\small
\begin{tabularx}{\textwidth}{l X c c c c}
\toprule
\textbf{Error} & \textbf{Model} & \multicolumn{4}{c}{\textbf{$L_0$ (number of inputs perturbed)}} \\
\cmidrule(lr){3-6}
\textbf{Threshold} & & \textbf{1} & \textbf{3} & \textbf{5} & \textbf{10} \\
\midrule
\multirow{4}{*}{\textbf{10\%}} 
  & MIMONet        & 100   & 100   & 100    & 100 \\
  & NOMAD          & 87.0  & 100   & 100    & 100 \\
  & S-DeepONet     & 100   & 100   & 100    & 100 \\
  & POD-DeepONet   & 100   & 100   & 100    & 100 \\
\midrule
\multirow{4}{*}{\textbf{20\%}} 
  & MIMONet        & 55.8  & 67.4  & 73.8   & 91.7 \\
  & NOMAD          & 54.2  & 80.9  & 91.0   & 100 \\
  & S-DeepONet     & 51.9  & 100   & 100    & 100 \\
  & POD-DeepONet   & 71.3  & 73.6  & 71.9   & 71.9 \\
\midrule
\multirow{4}{*}{\textbf{30\%}} 
  & MIMONet        & 9.4   & 13.1  & 22.6   & 36.1 \\
  & NOMAD          & 8.6   & 22.6  & 47.7   & 62.3 \\
  & S-DeepONet     & 17.1  & 94.5  & 100    & 100 \\
  & POD-DeepONet   & 24.5  & 25.2  & 28.4   & 29.6 \\
\midrule
\multirow{4}{*}{\textbf{40\%}} 
  & MIMONet        & 0     & 0     & 0      & 0 \\
  & NOMAD          & 0     & 0     & 2.2    & 17.1 \\
  & S-DeepONet     & 0     & 77.7  & 93.5   & 99.0 \\
  & POD-DeepONet   & 0     & 0     & 0      & 0 \\
\bottomrule
\end{tabularx}
\end{table}

\subsection{Field-Level Impact of Adversarial Perturbations}

Figure~\ref{fig:field_comparison} demonstrates the catastrophic field-level consequences of sparse adversarial attacks. Each row presents clean prediction (left), attacked prediction (middle), and absolute error map (right) for S-DeepONet under $L_0=3$ attack and NOMAD under $L_0=1$ attack.

For S-DeepONet, all three velocity components exhibit severe spatial distortions despite the model maintaining globally reasonable field magnitudes. The velocity-Y error map reveals hotspot relocations with peak errors exceeding 100\% of the nominal range in localized regions, while velocity-X shows complete flow pattern inversion. The pressure field, which appears visually similar between clean and attacked predictions, exhibits highly localized error spikes concentrated near geometric features, precisely where critical safety margins are evaluated in nuclear applications.

NOMAD exhibits different failure modes under single-point attack. The perturbation induces spatially localized but high-magnitude errors concentrated in boundary layer regions. The pressure error distribution reveals a distinct sensitivity structure: errors manifest as radially symmetric patterns emanating from the perturbation location, consistent with NOMAD's multiplicative fusion that preserves spatial locality.

These field comparisons expose the inadequacy of aggregate validation metrics for adversarial robustness assessment. Both models achieve less than 2\% relative $L_2$ error on clean test data, yet produce field predictions with over 30\% local errors and complete flow topology inversions under sparse attacks. The error maps demonstrate that safety-critical quantities (hotspot locations, peak temperatures, minimum safety margins) can be severely misrepresented even when global error metrics remain within acceptable tolerances.

\begin{figure}[p]
    \centering
    \includegraphics[width=\textwidth]{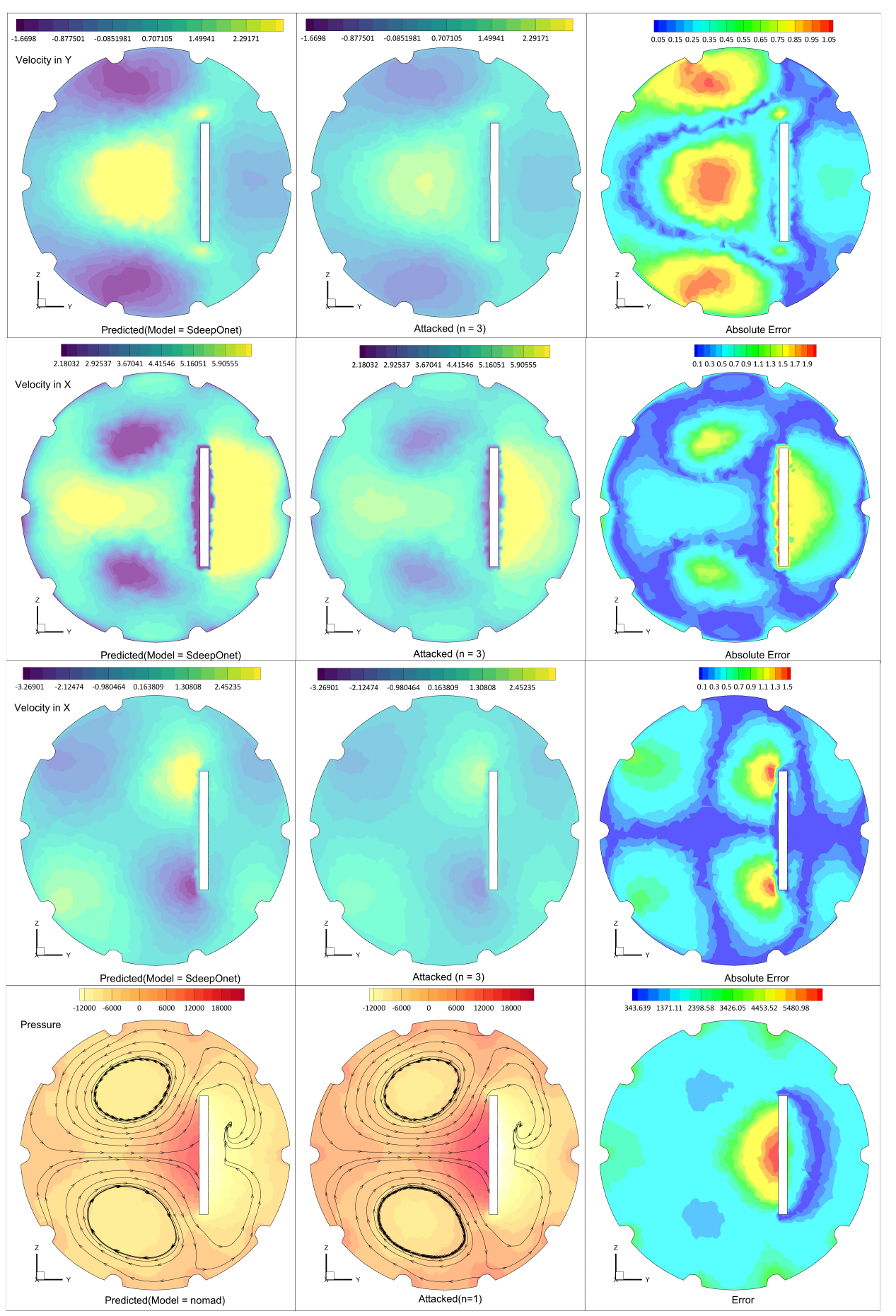}
    \caption{\textbf{Field-level impact of sparse adversarial attacks.} 
    Each row shows clean prediction (left), attacked prediction (middle), and absolute error (right). 
    \textbf{Rows 1--3}: S-DeepONet velocity components under $L_0=3$ attack. 
    \textbf{Row 4}: NOMAD pressure field under $L_0=1$ attack. 
    Error maps reveal localized failures with normalized errors approaching 1.0 in safety-critical regions, despite models achieving $<2\%$ error on clean data. 
    White rectangles indicate geometric obstacles.}
    \label{fig:field_comparison}
\end{figure}

\subsection{Explaining Differential Vulnerability: The Effective Perturbation Dimension}

\subsubsection{Quantifying Sensitivity via the Input-Output Jacobian}

The vulnerability patterns in Table~\ref{tab:attack_success} can be mechanistically explained through the structure of each model's input-output Jacobian. Consider a neural operator $f: \mathbb{R}^d \rightarrow \mathbb{R}^m$ mapping $d$-dimensional branch inputs $\mathbf{b} \in \mathbb{R}^d$ to $m$-dimensional outputs (field values at all spatial query points). The Jacobian matrix $J_b \in \mathbb{R}^{m \times d}$ is defined as:
\begin{equation}
J_b = \frac{\partial f}{\partial \mathbf{b}}, \quad \text{where} \quad [J_b]_{ji} = \frac{\partial f_j}{\partial b_i}.
\end{equation}
Each column $J_b[:,i] \in \mathbb{R}^m$ represents the sensitivity of the \emph{entire output field} to perturbations in the $i$-th input coordinate. The column norm $s_i = \|J_b[:,i]\|_2$ thus quantifies the aggregate influence of input coordinate $i$ on the predicted solution field.

It is the \emph{distribution} of these sensitivities $\{s_i\}_{i=1}^d$ across input coordinates, rather than their aggregate magnitude, that determines susceptibility to sparse attacks. A model where sensitivity concentrates in few coordinates can be severely disrupted by perturbing only those coordinates, while a model with uniformly distributed sensitivity requires coordinated perturbations across many inputs.

\subsubsection{Effective Perturbation Dimension}

To quantify sensitivity concentration, we introduce the \textbf{effective perturbation dimension} $d_{\text{eff}}$:
\begin{equation}
\label{eq:d_eff}
d_{\text{eff}} = \frac{\left(\sum_{i=1}^{d} \|J_b[:,i]\|_2\right)^2}{\sum_{i=1}^{d} \|J_b[:,i]\|_2^2} = \frac{\left(\sum_{i=1}^{d} s_i\right)^2}{\sum_{i=1}^{d} s_i^2}.
\end{equation}
This metric is the inverse of the Herfindahl-Hirschman concentration index applied to the sensitivity distribution, and admits a natural interpretation as the ``effective number of sensitive input directions.'' The metric satisfies $1 \leq d_{\text{eff}} \leq d$, with boundary cases:
\begin{itemize}
    \item $d_{\text{eff}} = 1$: All sensitivity concentrated in a single input coordinate (maximum vulnerability to single-point attacks).
    \item $d_{\text{eff}} = d$: Sensitivity uniformly distributed across all coordinates (maximum robustness to sparse attacks).
\end{itemize}

\paragraph{Illustrative Examples.} Consider a model with $d=100$ input coordinates:
\begin{enumerate}
    \item \textit{Single dominant direction}: If $s_1 = 1$ and $s_i = 0$ for $i > 1$, then $d_{\text{eff}} = 1^2/1^2 = 1$.
    \item \textit{Two equally sensitive directions}: If $s_1 = s_2 = 1$ and $s_i = 0$ for $i > 2$, then $d_{\text{eff}} = (1+1)^2/(1^2+1^2) = 2$.
    \item \textit{Uniform sensitivity}: If $s_i = 1$ for all $i$, then $d_{\text{eff}} = (100)^2/(100 \cdot 1^2) = 100$.
\end{enumerate}
Thus, $d_{\text{eff}}$ directly indicates how many input coordinates an adversary must perturb to achieve substantial output degradation. We prove formally (Appendix~\ref{app:theory}, Theorem~\ref{thm:sparse_adv_a}) that the effectiveness of a single-point attack scales as $\rho(1) \geq 1/\sqrt{d_{\text{eff}}}$, where $\rho(k)$ is the fraction of dense attack damage achievable by a $k$-sparse attack, establishing a rigorous quantitative link between $d_{\text{eff}}$ and adversarial vulnerability.

\subsubsection{Computing $d_{\text{eff}}$ for Neural Operators}

For each model architecture, we compute $d_{\text{eff}}$ by evaluating the Jacobian via automatic differentiation. Given a trained model $f$, branch inputs $\mathbf{b} = [\mathbf{b}_1; \mathbf{b}_2] \in \mathbb{R}^{102}$ (concatenating the 2-dimensional and 100-dimensional branch inputs), and trunk coordinates $\mathbf{x} \in \mathbb{R}^{N \times 2}$, we compute:
\begin{equation}
J_b = \frac{\partial}{\partial \mathbf{b}} \left[ f(\mathbf{b}, \mathbf{x}) \right] \in \mathbb{R}^{(N \cdot C) \times 102},
\end{equation}
where $N$ is the number of spatial query points and $C$ is the number of output channels. The effective dimension is then computed from the column norms of $J_b$ using Eq.~\eqref{eq:d_eff}. We report the mean $d_{\text{eff}}$ averaged over all test samples to characterize each architecture's intrinsic sensitivity structure.

Table~\ref{tab:d_eff} summarizes the measured effective perturbation dimensions:

\begin{table}[h]
\centering
\caption{\textbf{Effective perturbation dimension across architectures.} Lower $d_{\text{eff}}$ indicates concentrated sensitivity. Mean $\pm$ std computed over 50 test samples using randomized Jacobian projections ($n_{\text{proj}}=30$).}
\label{tab:d_eff}
\begin{tabular}{lccc}
\toprule
\textbf{Model} & $\mathbf{d_{\text{eff}}}$ & \textbf{Mean col.\ norm} & \textbf{Interpretation} \\
\midrule
POD-DeepONet   & $1.02 \pm 0.01$  & $3 \times 10^{-4}$ & Extreme concentration, low magnitude \\
S-DeepONet     & $4.35 \pm 0.50$  & $8 \times 10^{-4}$ & Concentrated, moderate magnitude \\
MIMONet        & $31.93 \pm 5.39$ & $2.3 \times 10^{-3}$ & Distributed, high magnitude \\
NOMAD          & $31.32 \pm 4.98$ & $2.2 \times 10^{-3}$ & Distributed, high magnitude \\
\bottomrule
\end{tabular}
\end{table}

\begin{figure}[htbp]
\centering
\includegraphics[width=\textwidth, height=0.85\textheight, keepaspectratio]{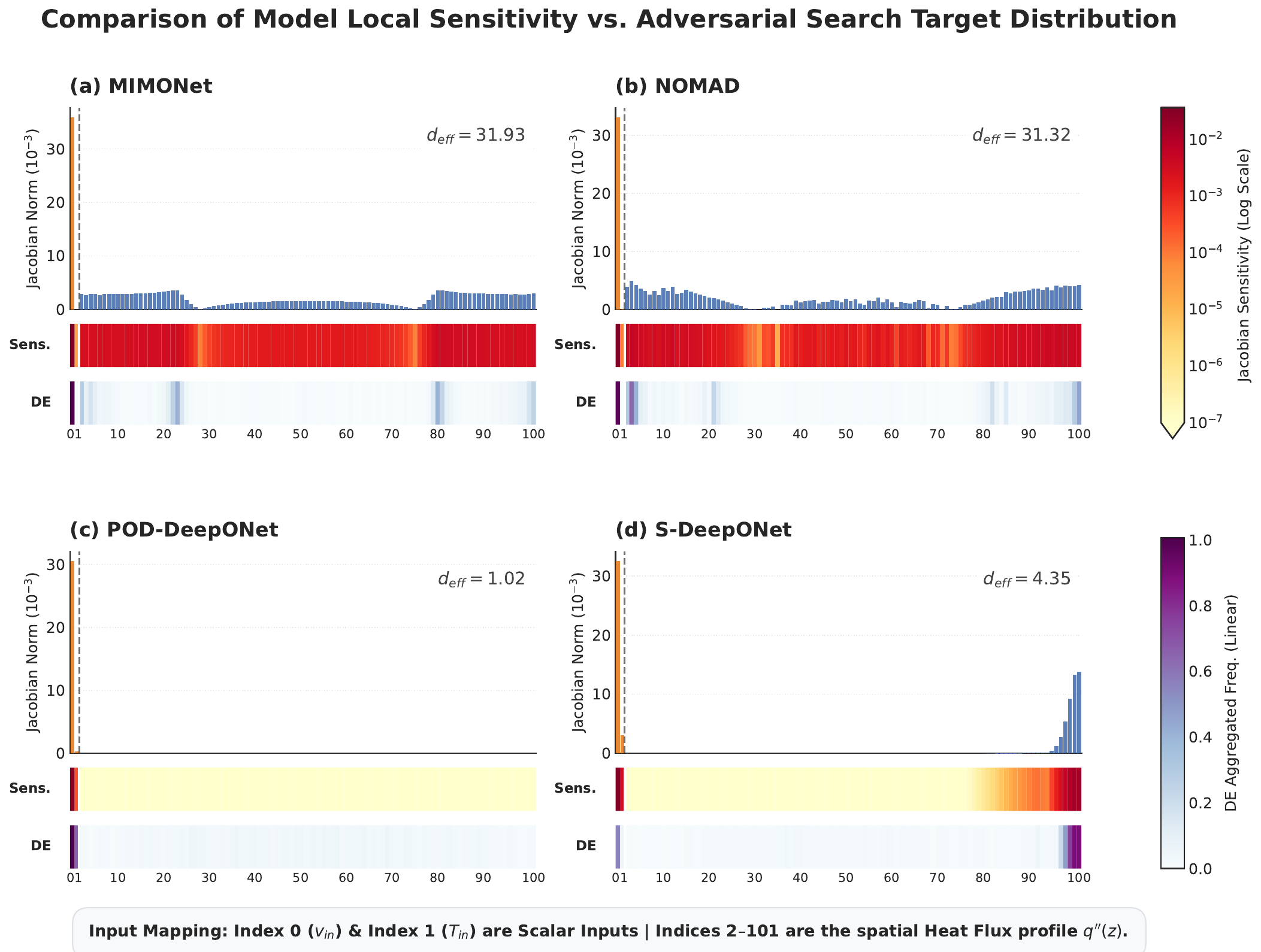}
\caption{\textbf{Jacobian sensitivity profiles and DE attack targeting across architectures.}
\emph{Top panels (bars):} Mean Jacobian column norms $s_i = \|J_b[:,i]\|_2$ ($\times 10^{-3}$) for each of the 102 input coordinates, averaged over 50 test samples with 30 randomized projections each. Dashed red line marks the Branch~1/Branch~2 boundary (index 2). \emph{Middle strips (Sensitivity):} Same data as a log-scale heatmap for cross-model comparison. \emph{Bottom strips (DE Target):} Normalized DE attack targeting frequency, aggregated across $L_0 \in \{1,3,5,10\}$ weighted by successful attacks. S-DeepONet's targeting concentrates at indices 0 and 99--101, matching its Jacobian peaks ($r = 0.99$). POD-DeepONet's apparently diffuse Branch~2 targeting reflects random budget-filling at higher $L_0$ rather than genuine vulnerability: Branch~2 Jacobian norms are ${\sim}10^{-12}$, and targeting is uncorrelated with sensitivity ($r = -0.47$); see Section~\ref{sec:arch_vulnerability} and Fig.~\ref{fig:supp_heatmaps} for details.}
\label{fig:jacobian_profiles}
\end{figure}

\subsubsection{Architecture-Specific Vulnerability Mechanisms}
\label{sec:arch_vulnerability}

The measured $d_{\text{eff}}$ values reveal a more complex relationship between sensitivity concentration and attack success than a simple monotonic prediction would suggest. Critically, the vulnerability of a model depends on the \emph{product} of two factors: the concentration of sensitivity (captured by $d_{\text{eff}}$) and the magnitude of sensitivity (captured by mean Jacobian column norms). We observe three distinct vulnerability regimes (Fig.~\ref{fig:jacobian_profiles}).

\paragraph{POD-DeepONet: Extreme Concentration, Low Magnitude ($d_{\text{eff}} \approx 1.0$).}
POD-DeepONet exhibits the most extreme sensitivity concentration of any architecture, with $d_{\text{eff}} \approx 1.02 \pm 0.01$, meaning that virtually all output sensitivity is confined to a single input direction. Yet this extreme concentration does not translate to the highest attack success rates (Table~\ref{tab:attack_success}): at $L_0=3$ and 30\% threshold, POD-DeepONet reaches only 25.2\%, compared to S-DeepONet's 94.5\%.

This apparent paradox is resolved by examining the sensitivity \emph{magnitude}: POD-DeepONet's mean Jacobian column norm ($3 \times 10^{-4}$) is an order of magnitude smaller than other architectures. The POD basis projection acts as a double-edged constraint: it concentrates all sensitivity into the dominant mode direction (explaining $d_{\text{eff}} \approx 1$) while simultaneously \emph{limiting the maximum achievable error} by restricting outputs to a low-rank subspace of rank $r=32$. Perturbations along the single sensitive direction can shift POD coefficients, but the reconstruction $f(\mathbf{b}, \mathbf{x}) = \mathbf{U} \cdot g(\mathbf{b})$ projects the result back onto smooth basis functions, preventing the large-scale field distortions observed in less constrained architectures. This yields a characteristic plateau in POD-DeepONet attack success across $L_0$: increasing the perturbation budget from 1 to 10 produces only marginal improvement (24.5\% to 29.6\% at 30\% threshold), as additional perturbations target directions with negligible sensitivity. The attack targeting data (Fig.~\ref{fig:jacobian_profiles}, bottom strips) corroborate this mechanistically: POD-DeepONet's Branch~2 Jacobian norms are effectively zero ($\sim 10^{-12}$, twelve orders of magnitude below Branch~1), so DE's Branch~2 selections at higher $L_0$ are uncorrelated with the Jacobian landscape ($r = -0.47$ between Branch~2 column norms and DE targeting frequency). The near-maximum entropy of POD-DeepONet's Branch~2 frequency distribution (96--99\% of the uniform-distribution entropy at $L_0 = 3$--$10$) confirms that these selections are effectively random filler rather than targeted exploitation.

\paragraph{S-DeepONet: Moderate Concentration, High Amplification ($d_{\text{eff}} \approx 4.4$).}
S-DeepONet is the most practically vulnerable architecture, achieving 94.5\% success at $L_0=3$ and 100\% at $L_0=5$ for the 30\% threshold, with mean errors reaching 50\% at $k=10$ (Table~\ref{tab:attack_success}). Its $d_{\text{eff}} \approx 4.35 \pm 0.50$ indicates that output sensitivity concentrates in approximately four to five input directions, while its moderate column norm magnitude ($8 \times 10^{-4}$) enables sufficient error amplification.

Architecturally, S-DeepONet's vulnerability arises from its GRU-based sequential branch encoder, which processes the 100-point wall heat flux profile as a temporal sequence. The GRU hidden state accumulates spatial dependencies such that perturbations at specific positions propagate through the recurrent dynamics, affecting the final hidden state $\mathbf{h}_{100}$ that determines all output coefficients simultaneously. Unlike POD-DeepONet, S-DeepONet's trunk-branch composition ($\sum_k \beta_k(\mathbf{b}) \cdot \tau_k(\mathbf{x})$) operates without a regularizing projection, allowing perturbation-induced coefficient changes to produce unrestricted field distortions including the hotspot relocations and flow pattern inversions documented in Figure~\ref{fig:field_comparison}. The targeting structure confirms this: DE's Branch~2 selections for S-DeepONet are highly concentrated at indices 99--101 (the sequential endpoints) with a near-perfect correlation between Jacobian column norms and attack frequency ($r = 0.99$), in stark contrast to POD-DeepONet's random Branch~2 targeting ($r = -0.47$). This correlation confirms that DE independently discovers and exploits the same sensitivity structure that the Jacobian analysis identifies.

\paragraph{MIMONet and NOMAD: Distributed Sensitivity ($d_{\text{eff}} \approx 32$).}
MIMONet ($d_{\text{eff}} \approx 31.9 \pm 5.4$) and NOMAD ($d_{\text{eff}} \approx 31.3 \pm 5.0$) exhibit the most distributed sensitivity profiles, with influence spread across approximately one-third of all 102 input coordinates. Both architectures have the highest mean column norms ($\sim 2.2$--$2.3 \times 10^{-3}$), indicating individually strong input-output coupling, but the high dimensionality of their sensitive subspace makes sparse attacks substantially harder: at $L_0=1$, a single perturbation intercepts only $\sim$1/32 of the sensitive subspace, explaining the moderate 9.4\% (MIMONet) and 8.6\% (NOMAD) success rates at the 30\% threshold. Unlike POD-DeepONet's random Branch~2 scattering, DE's targeting for MIMONet and NOMAD shows moderate positive correlation with the Jacobian landscape ($r = 0.50$ and $r = 0.57$, respectively), indicating that even in these high-$d_{\text{eff}}$ architectures, DE preferentially targets the more sensitive coordinates rather than selecting randomly.

NOMAD's vulnerability increases more steeply with $L_0$ than MIMONet's (62.3\% vs.\ 36.1\% at $L_0=10$, 30\% threshold), consistent with NOMAD's multiplicative branch-trunk fusion amplifying coordinated multi-point perturbations more effectively than MIMONet's concatenation-based architecture. This differential scaling behavior demonstrates that even architectures with similar $d_{\text{eff}}$ can exhibit substantially different vulnerability profiles at higher sparsity budgets.

\subsection{Two-Factor Vulnerability Model: Concentration $\times$ Magnitude}
\label{sec:two_factor}

The effective perturbation dimension provides a necessary but not sufficient predictor for sparse attack vulnerability. Our measurements reveal that attack success depends on the interplay of two factors: (1)~sensitivity concentration ($d_{\text{eff}}$), which determines how easily sparse perturbations intersect the sensitive subspace, and (2)~sensitivity magnitude (mean Jacobian column norm), which determines the amplification from input perturbation to output error. This two-factor model explains the empirical observations:
\begin{itemize}
    \item POD-DeepONet ($d_{\text{eff}} \approx 1$, low magnitude): Extreme concentration but insufficient amplification yields moderate vulnerability that plateaus with increasing $L_0$.
    \item S-DeepONet ($d_{\text{eff}} \approx 4$, moderate magnitude): Sufficient concentration \emph{and} amplification yields the highest vulnerability, especially at $L_0 \geq 3$.
    \item MIMONet/NOMAD ($d_{\text{eff}} \approx 32$, high magnitude): High amplification but distributed sensitivity requires coordinated multi-point attacks; success increases gradually with $L_0$.
\end{itemize}

This empirical two-factor model is formalized in Appendix~\ref{app:theory}: we prove that the optimal $k$-sparse attack error decomposes exactly as $\mathcal{E}^* = M \cdot \rho(k)$, where $M = \epsilon \|\mathbf{s}\|_2 / \|f(\mathbf{b})\|_2$ is the magnitude factor and $\rho(k) \geq 1/\sqrt{d_{\text{eff}}}$ is the concentration factor (Theorems~\ref{thm:two_factor_main} and~\ref{thm:sparse_adv_a}). This framework extends $d_{\text{eff}}$ from a standalone diagnostic to a theoretically grounded component of a composite vulnerability metric. For pre-deployment risk assessment, both $d_{\text{eff}}$ and sensitivity magnitude should be evaluated: models combining low $d_{\text{eff}}$ with high column norms represent the highest-risk configurations.

This analysis also reveals a fundamental tension in neural operator architecture: designs optimized purely for reconstruction accuracy can inadvertently develop concentrated sensitivity structures that adversaries exploit. POD-DeepONet's projection onto data-driven principal modes provides a form of implicit defense: while it concentrates sensitivity ($d_{\text{eff}} \approx 1$), the low-rank reconstruction limits the maximum achievable error, creating a natural ``ceiling'' on adversarial degradation. S-DeepONet's sequential encoding, while powerful for capturing boundary condition dependencies, creates amplified vulnerability through recurrent hidden state coupling. The concatenation-based architectures (MIMONet, NOMAD) naturally distribute sensitivity through their multi-pathway design but remain vulnerable to coordinated multi-point attacks.

These findings suggest several design principles for adversarially robust neural operators:
\begin{enumerate}
    \item \textbf{Sensitivity regularization}: Explicitly penalizing Jacobian column norm concentration during training (e.g., maximizing $d_{\text{eff}}$ as an auxiliary objective) while monitoring absolute sensitivity magnitude.
    \item \textbf{Output-space constraints}: Employing low-rank projections or spectral truncation (as in POD-DeepONet) to limit the maximum achievable perturbation amplification, even at the cost of some representational flexibility.
    \item \textbf{Deployment diagnostics}: Evaluating both $d_{\text{eff}}$ and mean column norms as a composite pre-deployment vulnerability assessment; models combining $d_{\text{eff}} < 5$ with high sensitivity magnitude should be considered highest-risk for safety-critical applications.
\end{enumerate}

The effective perturbation dimension thus provides both a diagnostic tool for evaluating existing models and a design criterion for developing inherently robust neural operator architectures.

\subsection{Corroborating Evidence and Robustness Checks}

\subsubsection{Per-Channel Error Analysis}

Adversarial perturbations produce highly non-uniform degradation across output channels. Under successful attacks at the 30\% threshold, the pressure field ($P$) consistently exhibits the largest relative $L_2$ errors across all architectures (53--73\%), while velocity components show lower but still substantial degradation: axial velocity $w$ (35--55\%), and transverse velocities $u$, $v$ (15--49\%). This pattern is consistent across sparsity budgets and reflects the differing spectral characteristics of each field: pressure is globally coupled through the Poisson equation and thus more sensitive to boundary condition perturbations, while velocity fields are partially protected by local momentum balance constraints. S-DeepONet exhibits the most uniform cross-channel degradation (Fig.~\ref{fig:supp_error_distributions}) at higher sparsity ($k \geq 5$), with all channels exceeding 46\% error, indicating that its vulnerability mechanism (GRU hidden state corruption) disrupts the entire latent representation rather than targeting specific output modes (complete per-channel breakdown in Table~\ref{tab:supp_perchannel_attack}).

\begin{figure}[t]
\centering
\includegraphics[width=\textwidth]{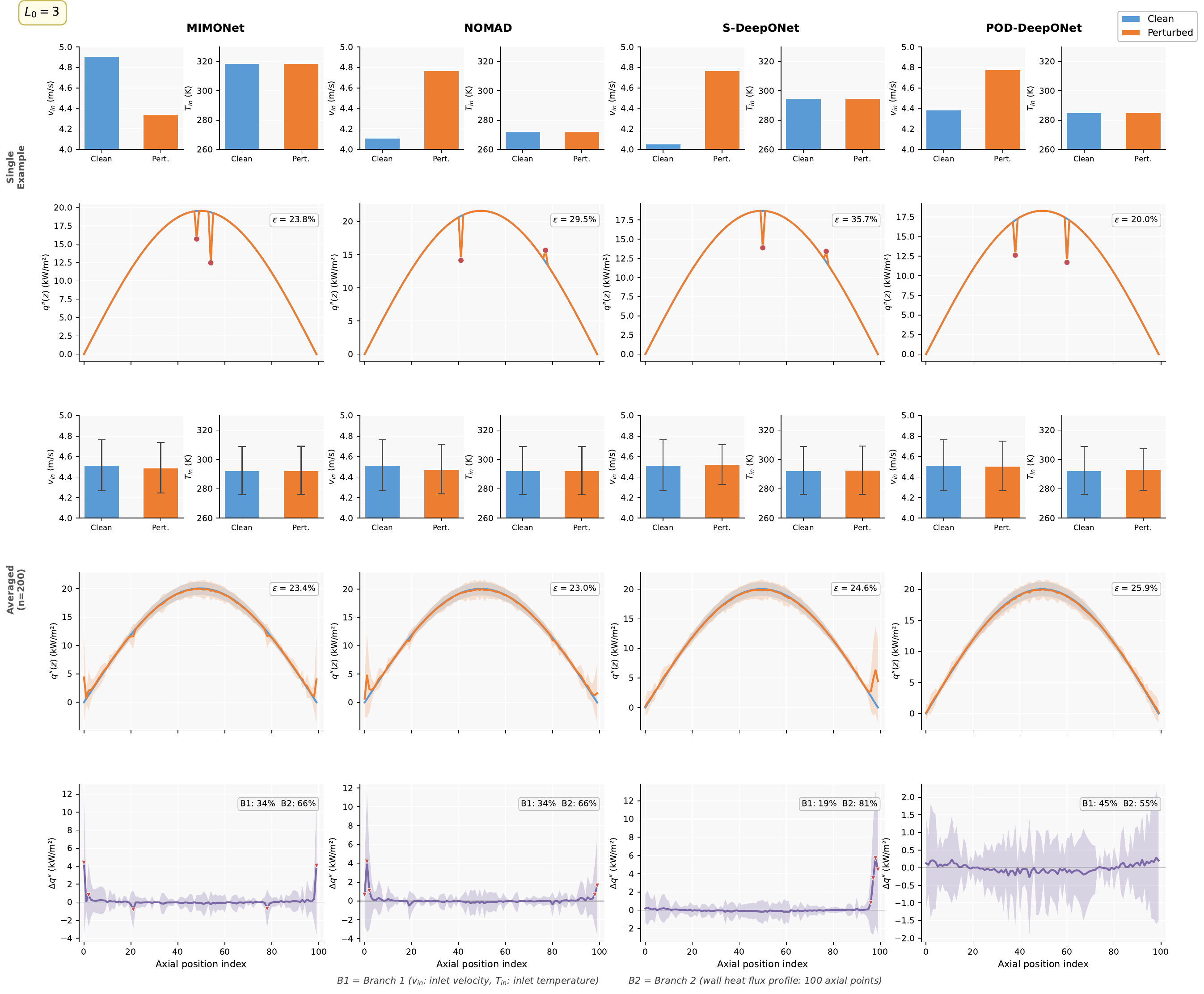}
\caption{\textbf{Visualizing adversarial perturbations at $L_0=3$.}
To illustrate what successful adversarial inputs look like in practice, we plot both a representative single attack (rows 1--2) and averaged statistics across 200 successful attacks (rows 3--5) for each architecture at $L_0=3$ with 20\% error threshold. All values are in physical units.
\textbf{Rows 1 and 3}~(Branch~1): Inlet velocity $v_{\text{in}}$ (m\,s$^{-1}$) and temperature $T_{\text{in}}$ (K) before and after perturbation.
\textbf{Rows 2 and 4}~(Branch~2): Wall heat flux profiles $q''(z)$ in kW\,m$^{-2}$ along the 100 axial positions; red circles mark where perturbations were applied in the single examples.
\textbf{Row 5}: Perturbation difference ($\Delta q''$) in kW\,m$^{-2}$ showing where attacks concentrate along the boundary. Interestingly, S-DeepONet attacks cluster at the right boundary (positions 97--99), while MIMONet and NOMAD target both boundaries more symmetrically. A normalized version is provided in Fig.~\ref{fig:supp_perturbation_normalized}.}
\label{fig:perturbation_structure}
\end{figure}

\subsubsection{Branch Vulnerability Structure}

The 102-dimensional input space decomposes into Branch~1 (2 global parameters: $v_{\text{in}}$, $T_{\text{in}}$) and Branch~2 (100 boundary heat flux values). Analysis of successful attack genomes reveals a clear asymmetry: at $L_0=1$, all architectures except S-DeepONet exclusively target Branch~1 (100\% of successful attacks perturb indices 0 or 1), indicating that the two global parameters, representing less than 2\% of the input space, are disproportionately influential. S-DeepONet is the exception, with 79\% Branch~1-only and 21\% Branch~2-only attacks at $L_0=1$, consistent with its sequential encoder providing non-trivial sensitivity to individual heat flux values.

To see what these perturbations actually look like, Figure~\ref{fig:perturbation_structure} visualizes the attack data at $L_0=3$ (a representative mid-case with 200 successful attacks per model). The top rows show a single concrete example for each architecture, while the bottom rows aggregate statistics across all samples. For MIMONet and NOMAD, the attacks visibly shift $v_{\text{in}}$ upward (from ${\sim}4.5$ to ${\sim}4.8$~m\,s$^{-1}$ in the examples), with the heat flux profiles showing sparse deviations at boundary positions. S-DeepONet and POD-DeepONet attacks, by contrast, leave Branch~1 largely unchanged and concentrate perturbations in Branch~2. Looking at the averaged perturbation difference (bottom row), we observe that MIMONet and NOMAD show symmetric boundary targeting (B1: 34\%, B2: 66\%), while S-DeepONet exhibits a pronounced right-boundary concentration at positions 97--99 (B1: 19\%, B2: 81\%). This asymmetry is consistent with the GRU architecture: since the encoder processes heat flux values sequentially, perturbations at the final positions (sequence endpoints) have the strongest effect on the hidden state that determines all output coefficients. POD-DeepONet shows minimal structured perturbation in Branch~2; the noise-like pattern reflects random budget-filling after the single sensitive direction (Branch~1) is saturated, consistent with its $d_{\text{eff}} \approx 1$.

At $L_0 \geq 3$, attacks become predominantly Mixed (involving both branches), with S-DeepONet again distinguished by substantial Branch~2-only attacks (40--51\% of successes at $k=3$--$10$). The near-exclusive Branch~1 targeting at $L_0=1$ across MIMONet, NOMAD, and POD-DeepONet suggests that global operating parameters create a universal high-leverage attack surface, informing defensive priorities: hardening the Branch~1 input pathway (sensor integrity for inlet velocity and temperature) would mitigate the most accessible single-point attacks (Table~\ref{tab:supp_branch_vuln}).

\subsubsection{Error Contextualization}

To contextualize the severity of adversarial degradation, we note that the maximum observed errors under attack (38.3\% for MIMONet, 47.7\% for NOMAD, 37.0\% for POD-DeepONet, and 62.8\% for S-DeepONet) represent a collapse from the baseline $\sim$1.5\% test error by factors of 25$\times$ to 42$\times$. S-DeepONet's worst-case error of 62.8\%, where the adversarial prediction differs from the clean prediction by nearly two-thirds of the field norm, corresponds to predictions that are not merely inaccurate but physically meaningless for safety-critical decision-making.

\subsubsection{Statistical Confidence}

All reported success rates are accompanied by Wilson score 95\% confidence intervals (Table~\ref{tab:supp_wilson_ci}). For the primary results at the 30\% threshold, representative intervals include: MIMONet $L_0=5$: 22.6\% [18.3\%, 27.6\%]; NOMAD $L_0=5$: 47.7\% [42.2\%, 53.3\%]; S-DeepONet $L_0=3$: 94.5\% [91.4\%, 96.5\%]; POD-DeepONet $L_0=3$: 25.2\% [20.7\%, 30.3\%]. The tight confidence intervals confirm that the observed differential vulnerability patterns are statistically robust and not artifacts of finite sample sizes.

\subsubsection{Random Perturbation Baseline}

To distinguish DE-discovered vulnerabilities from accidental sensitivity, we compare against random $k$-sparse perturbations drawn uniformly from $[-1, +1]$ in standardized space (50 random trials per sample, same bounds and sparsity as DE). At the 30\% threshold, random perturbations achieve success rates of 0.3--0.6\% for MIMONet, 0--1.6\% for NOMAD, 0.6--8.7\% for POD-DeepONet, and 0.3--8.7\% for S-DeepONet across $L_0 \in \{1, 3, 5, 10\}$ (Table~\ref{tab:supp_random_baseline}; Fig.~\ref{fig:random_baseline}). These are 10--100$\times$ lower than DE success rates at the same configurations (e.g., S-DeepONet $L_0=3$: random 1.0\% vs.\ DE 94.5\%; NOMAD $L_0=10$: random 1.6\% vs.\ DE 62.3\%). The near-zero random success rates, despite perturbations satisfying identical physical bounds, confirm that DE exploits structured vulnerability subspaces rather than inherent input-output instability, and that the high attack success rates reported in Table~\ref{tab:attack_success} reflect genuine adversarial optimization, not trivial sensitivity.

\begin{figure}[t]
\centering
\includegraphics[width=\textwidth]{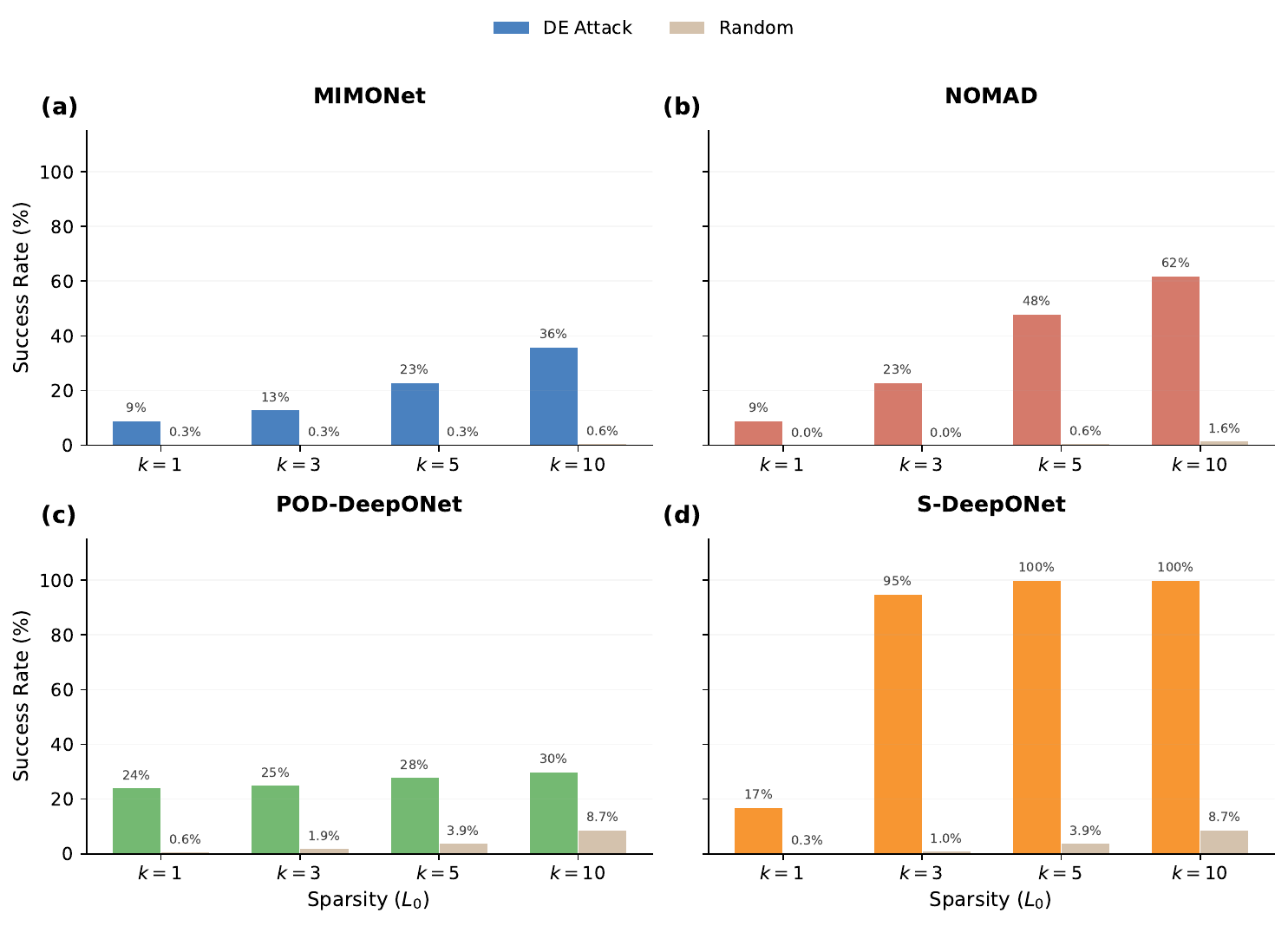}
\caption{\textbf{DE vs.\ random perturbation baseline.}
Comparison of DE (blue) and random $k$-sparse perturbation (gray) success rates at the 30\% threshold over 310 test samples (50 random trials per sample). Random perturbations achieve $<$9\% across all configurations while DE reaches 9--100\%, confirming that DE exploits structured vulnerability subspaces rather than inherent input-output instability.}
\label{fig:random_baseline}
\end{figure}
\subsubsection{Cross-Architecture Transferability}

Adversarial perturbations crafted for one architecture exhibit limited transferability to others (Table~\ref{tab:supp_transfer}). Using successful $L_0=1$ attack genomes from each source model, we evaluate cross-architecture transfer success at the 30\% threshold. S-DeepONet perturbations retain 83\% self-success but transfer to other architectures at only 2--6\%, demonstrating that S-DeepONet's vulnerability is highly architecture-specific: its GRU-mediated sensitivity structures are not shared by the other designs. MIMONet and NOMAD perturbations show higher cross-transfer rates (9--36\%). POD-DeepONet perturbations exhibit moderate cross-transfer (6--12\%). The overall pattern of architecture-specific vulnerabilities with low transfer success has important defensive implications: ensemble methods combining architecturally diverse operators would provide natural robustness, as an adversarial input optimized for one model is unlikely to simultaneously fool alternatives.

\subsubsection{Validation Metric Blindness}

A critical finding for deployment safety is that the most accessible adversarial attacks (single-point perturbations) are entirely invisible to standard output validation. For all four models, 100\% of successful $L_0=1$ attacks produce outputs whose z-scores (relative to the clean test set distribution) remain below the conventional anomaly threshold of 3$\sigma$ (Table~\ref{tab:supp_validation_blind}; Fig.~\ref{fig:supp_zscore}). At $L_0=1$, the mean input-to-output amplification ratio reaches 5.2$\times$ for MIMONet, 5.9$\times$ for NOMAD, 7.9$\times$ for POD-DeepONet, and 7.5$\times$ for S-DeepONet, meaning that a perturbation affecting less than 1\% of the input vector produces 5--8$\times$ larger relative change in the 15,908-dimensional output field (Fig.~\ref{fig:supp_amplification}). At higher sparsity ($L_0 \geq 3$), z-score pass rates drop to 0--28\%, and amplification ratios fall below 1$\times$ as the relative input perturbation magnitude grows faster than the output error increase: the attacker modifies more of the input but encounters diminishing returns in output degradation (Table~\ref{tab:supp_validation_blind}). The combination of high amplification, complete z-score invisibility, and physical constraint satisfaction at $L_0=1$ represents the most dangerous operational scenario: a single corrupted sensor reading can induce grossly incorrect predictions that pass all conventional quality checks.

\subsubsection{Reproducibility and Seed Sensitivity}

Attack outcomes are highly reproducible across random seeds. Over 5 independent DE runs per (model, $k$) configuration, the cross-seed standard deviation in success rate is at most 2.8 percentage points (NOMAD, $L_0=10$) and typically below 1.6 points, while the cross-seed standard deviation in mean error is uniformly below 0.015 (Table~\ref{tab:supp_seed}). POD-DeepONet shows the most deterministic behavior, with zero seed variance for $k \geq 3$, consistent with its $d_{\text{eff}} \approx 1$ concentrating all sensitivity in a single direction that DE reliably discovers regardless of initialization. These results confirm that the vulnerability patterns reported here are intrinsic model properties, not artifacts of random initialization.

\section{Discussion}\label{sec:discussion}

Our findings reveal that neural operator surrogates, despite demonstrating production-grade accuracy on validation datasets, exhibit systematic vulnerabilities to physics-compliant adversarial perturbations of extreme sparsity. Single-feature modifications (affecting fewer than 1\% of inputs) can induce order-of-magnitude accuracy degradation while satisfying all physical constraints, exposing a critical blind spot in current validation practices that assess surrogate quality through global metrics rather than worst-case local deviations.

The two-factor vulnerability model (Section~\ref{sec:two_factor}) has implications beyond diagnostic use. It reveals that the relationship between model design and adversarial robustness is richer than a simple sensitivity ranking: neither concentration alone ($d_{\text{eff}}$) nor magnitude alone (column norms) predicts vulnerability. The most operationally concerning finding is that the ``safest'' architecture by naive assessment (POD-DeepONet, with $d_{\text{eff}} \approx 1$) owes its relative resilience not to robustness \emph{per se} but to an incidental ceiling imposed by low-rank output projection. Architectures without such structural constraints (S-DeepONet, NOMAD) offer no comparable backstop. Critically, a regime combining low $d_{\text{eff}}$ with high sensitivity magnitude, not represented in our four models but readily achievable with poorly regularized encoders, would constitute the most dangerous configuration for safety-critical deployment. The two-factor framework thus serves not only as a post-hoc diagnostic but as a design constraint: future operator architectures intended for adversarial environments should be evaluated against both axes before deployment. More fundamentally, we prove an impossibility result (Appendix~\ref{app:theory}, Theorem~\ref{thm:impossibility_main}): any neural operator achieving Sobolev $\varepsilon$-accuracy for a sensitivity-varying operator necessarily inherits the operator's sensitivity structure, so that adversarial error scales as $\epsilon \cdot s_{(1)}^{\mathcal{G}} \gg \varepsilon$ as the approximation improves. This establishes that the observed vulnerabilities are not engineering failures but intrinsic consequences of faithful operator approximation, connecting our findings to the broader accuracy--robustness tradeoff literature~\cite{tsipras2019robustness}.

While gradient-based attacks (FGSM, PGD) could potentially achieve comparable success rates, they require white-box access to model internals and gradients, a strong assumption rarely satisfied for industrial digital twins where models are often proprietary black boxes. Our gradient-free approach (Fig.~\ref{fig:DE_encoding}; Fig.~\ref{fig:supp_de_pipeline}) provides a more realistic threat model and achieves comparable wall-clock time (approximately 2 minutes per attack) despite higher numbers of model evaluations, as computing gradients for 15,908-dimensional outputs imposes significant overhead (Appendix~\ref{app:note_gradfree}). Convergence analysis over 20 independent DE runs per configuration confirms that the attack optimization is highly reliable: S-DeepONet exceeds 30\% error within approximately 35 generations at $L_0=1$, while MIMONet and NOMAD require 80--150 generations at higher sparsity budgets (Fig.~\ref{fig:supp_convergence}).

Direct comparison with projected gradient descent (PGD, 100 iterations, step size $10^{-2}$, random restarts) at the 30\% threshold reveals that DE consistently outperforms PGD on architectures with gradient pathologies (Table~\ref{tab:supp_pgd_de}; Fig.~\ref{fig:supp_pgd_de}). For POD-DeepONet, DE achieves 34\% success at all sparsity levels while PGD reaches only 0--6\%, reflecting POD-DeepONet's low-rank gradient structure that confounds gradient-based search. For S-DeepONet, DE achieves 94--100\% success at $L_0 \geq 3$ versus PGD's 6--44\%, with PGD's wall-clock time ($\sim$52\,s per attack) an order of magnitude slower than for other models due to cuDNN compatibility constraints. On NOMAD, PGD performs competitively (82\% at $L_0=10$ vs.\ DE's 66\%), consistent with NOMAD's well-conditioned gradient landscape. These results confirm that gradient-free optimization provides critical advantages for architectures where gradient quality is compromised, precisely the models most relevant for black-box deployment scenarios.

For digital twins governing nuclear operations or grid stability, these findings challenge the assumption that simulation-validated models with demonstrated interpolation accuracy are inherently trustworthy~\cite{kobayashi2024trustworthy} in operational environments where sensor drift, calibration errors, or adversarial tampering can introduce physics-compliant but malicious input perturbations. The near-total invisibility of single-point attacks to standard z-score validation (100\% pass rate across all models) represents a particularly acute threat: the most accessible attacks, requiring corruption of only a single sensor reading, produce 5--8$\times$ amplification from input to output while remaining statistically indistinguishable from nominal predictions. Conversely, the low cross-architecture transferability (2--36\% transfer success) suggests that ensemble methods combining architecturally diverse models could provide effective defense, as adversarial inputs optimized for one model are unlikely to simultaneously fool alternatives. The model-agnostic framework provided here enables systematic robustness evaluation before deployment through black-box stress-testing under realistic attack scenarios.

Potential defense mechanisms warrant investigation across multiple fronts. We emphasize that the physical feasibility of all perturbations is guaranteed by construction (Table~\ref{tab:physical_bounds}): the $\pm 1\sigma$ search bounds ensure that perturbed inputs fall within the convex hull of training data and pass standard univariate and multivariate anomaly detection, making conventional input validation insufficient as a defense. Several complementary strategies merit investigation:

\begin{enumerate}
    \item \textbf{Architectural ensembles.} Combining predictions from diverse operator families could provide robustness through architectural diversity. Our transferability analysis shows that cross-architecture transfer success is only 2--36\%, providing empirical support: adversarial inputs optimized for one model are unlikely to simultaneously fool architecturally distinct alternatives.
    \item \textbf{Physics-based anomaly detection.} Monitoring energy balance, conservation laws, and cross-field consistency could flag implausible input combinations even when individual sensor readings remain within nominal bounds. Unlike univariate threshold checks, physics-informed detectors can exploit the governing PDE structure to identify inputs that are individually plausible but jointly inconsistent.
    \item \textbf{Sensitivity regularization during training.} Maximizing $d_{\text{eff}}$ while minimizing column norm magnitude as joint auxiliary objectives could distribute vulnerability across the input space while reducing amplification, directly targeting both factors of the vulnerability model.
    \item \textbf{Formal verification.} Establishing certified robustness guarantees for restricted input regions could delineate safe operating envelopes within which worst-case prediction error is provably bounded.
\end{enumerate}

The architecture-dependent vulnerability patterns observed here suggest that no single defense strategy will suffice; robustness must be co-designed with operator architecture from the outset rather than retrofitted post-deployment. Developing and evaluating such defenses is the focus of ongoing work.

\paragraph{Scope and Limitations.}
Our analysis uses a single thermal-hydraulic benchmark (heat exchanger with mixed boundary conditions), motivated by the computational cost of training four architectures to production-grade accuracy and executing over 4,960 adversarial evaluations. We mitigate this limitation through breadth across architectures: the four operator families span fundamentally different design choices (concatenation vs.\ operator decomposition, feedforward vs.\ recurrent encoding, full-field vs.\ reduced-order output), and the vulnerability mechanisms we identify (sensitivity concentration, GRU hidden-state coupling, low-rank output capping) are architecture-level properties rather than artifacts of the specific PDE. Nonetheless, confirming that these vulnerability patterns generalize to other PDE families, higher-dimensional input spaces, and mesh resolutions remains important future work. Additionally, our $\pm 1\sigma$ perturbation bound was chosen to guarantee stealth (all perturbed inputs lie within the training distribution's convex hull); relaxing this to $\pm 2\sigma$ or $\pm 3\sigma$ would likely increase attack success rates but at the cost of reduced detection invisibility, representing a natural attacker trade-off between efficacy and stealth that merits systematic exploration. Similarly, while we focus on characterizing the attack surface (consistent with the adversarial ML literature, where attack papers and defense papers are typically separate contributions), practical deployment will require the defense mechanisms discussed above.

\section{Conclusion}\label{sec:conclusion}

We have demonstrated that neural operators deployed in safety-critical digital twins are systematically vulnerable to sparse, physics-compliant adversarial perturbations that expose an attack surface absent from traditional validation protocols. The two-factor vulnerability model combining effective perturbation dimension $d_{\text{eff}}$ with sensitivity magnitude provides both a mechanistic explanation for differential architectural vulnerability and a practical pre-deployment diagnostic, supported by rigorous theoretical bounds proving that sparse attack effectiveness scales as $1/\sqrt{d_{\text{eff}}}$ and that accurate operator approximation inherently implies adversarial vulnerability (Appendix~\ref{app:theory}). Our analysis reveals that vulnerability arises not from a single architectural deficiency but from the interplay between sensitivity concentration and amplification: S-DeepONet's moderate concentration ($d_{\text{eff}} \approx 4$) coupled with sufficient amplification produces the most exploitable vulnerability, while POD-DeepONet's extreme concentration ($d_{\text{eff}} \approx 1$) is tempered by its low-rank output constraint. Our gradient-free differential evolution framework provides a practical tool for adversarial stress-testing that scales to high-dimensional field predictions and respects physical constraints. Future work should focus on extending this analysis to additional PDE families, developing architecture-aware defense mechanisms informed by the $d_{\text{eff}}$ diagnostic, and establishing robustness certification protocols specifically designed for operator learning models in cyber-physical applications. As neural operator surrogates are deployed on accelerated timelines into nuclear monitoring, grid management, and climate prediction pipelines, closing the gap between interpolation accuracy and adversarial robustness is not merely a research objective but a prerequisite for trustworthy deployment in systems where prediction failures carry physical consequences.

\section{Methods}\label{sec:methods}

\subsection{Neural Operator Architectures and Training}

\subsubsection{Dataset and Preprocessing}

Our computational fluid dynamics dataset comprises 1,546 heat exchanger simulations: 80\% training (1,236 samples) and 20\% test (310). Each sample includes Branch 1 (2D global parameters), Branch 2 (100D boundary discretization), and Trunk (3,977 spatial coordinates). Outputs comprise 4 field variables: pressure ($P$) and three velocity components ($u$, $v$, $w$). All inputs undergo standardized preprocessing with training set statistics (complete specifications in Appendix~\ref{app:dataset}).

\subsubsection{Model Architectures}

We compare four neural operator families representing distinct design philosophies: \textbf{NOMAD} (2.1M parameters) uses dual multilayer perceptron encoders with multiplicative branch-trunk fusion; \textbf{S-DeepONet} (1.8M parameters) employs a recurrent GRU encoder for sequential boundary dependencies with spectral summation composition; \textbf{POD-DeepONet} (1.6M parameters) leverages proper orthogonal decomposition to predict low-rank coefficients rather than full fields; \textbf{MIMONet} (1.2M parameters) uses direct concatenation-based mapping without explicit operator structure. All models achieve 1.3--1.5\% baseline test error (Table~\ref{tab:supp_baseline}). Complete architectural specifications, layer dimensions, and parameter counts are provided in Appendix~\ref{app:architectures}.

\subsubsection{Training Protocol}

All models were trained identically using mean squared error loss with Adam optimizer (learning rate $10^{-3}$, weight decay $10^{-6}$), ReduceLROnPlateau scheduling, batch size 4, and 10\% dropout. Models converged within 50 epochs to validation MSE of 0.012--0.014 on NVIDIA Tesla A100 GPUs (complete protocol in Appendix~\ref{app:training}).

\subsection{Sparse Adversarial Attack Methodology}

\subsubsection{Attack Formulation}

Given a trained neural operator $f_{\theta}$ and a test input $\tilde{\mathbf{b}} \in \mathbb{R}^d$ (in standardized coordinates) with clean prediction $\hat{\mathbf{s}}_{\text{clean}} = f_{\theta}(\tilde{\mathbf{b}})$, we seek a $k$-sparse adversarial input $\tilde{\mathbf{b}}^{\text{adv}}$ satisfying:
\begin{equation}
\begin{aligned}
\max_{\tilde{\mathbf{b}}^{\text{adv}}} \quad & \mathcal{E} = \frac{\|f_{\theta}(\tilde{\mathbf{b}}^{\text{adv}}) - \hat{\mathbf{s}}_{\text{clean}}\|_2}{\|\hat{\mathbf{s}}_{\text{clean}}\|_2} \\
\text{subject to} \quad & \|\tilde{\mathbf{b}}^{\text{adv}} - \tilde{\mathbf{b}}\|_0 \leq k, \quad \tilde{b}_i^{\text{adv}} \in [-1, 1] \;\;\forall\; i \in \mathcal{P},
\end{aligned}
\end{equation}
where $\mathcal{P} \subseteq \{1, \ldots, d\}$ with $|\mathcal{P}| \leq k$ denotes the set of perturbed coordinates, and the bound $[-1, 1]$ is enforced in standardized input space (see Section~\ref{sec:feasibility} for the physical feasibility guarantee). The sparsity budget $k$ corresponds to the $\ell_0$-norm constraint on the perturbation, following the standard convention in sparse adversarial attacks~\cite{croce2019sparse, kotyan2022adversarial}; throughout this paper, we use $L_0 = k$ interchangeably, with $L_0$ denoting the norm type and $k$ the specific budget value. Unperturbed coordinates retain their original values: $\tilde{b}_i^{\text{adv}} = \tilde{b}_i$ for $i \notin \mathcal{P}$. An attack succeeds if $\mathcal{E} > \tau$ where $\tau = 0.3$ (30\% threshold). This sparsity model captures realistic tampering: a single faulty sensor, a mis-set valve position, or a corrupted telemetry field.

\subsubsection{Differential Evolution Optimization}

Figure~\ref{fig:DE_encoding} illustrates the genome encoding and decoding scheme, while Fig.~\ref{fig:supp_de_pipeline} presents the complete algorithmic pipeline including physical feasibility enforcement. The encoding layer (Fig.~\ref{fig:DE_encoding}, top left; Fig.~\ref{fig:supp_de_pipeline}, panel~a) represents each candidate perturbation as $2k$ continuous values ($k$ index--value pairs specifying which input coordinates to modify and their replacement values) which are decoded into perturbed branch vectors. For each candidate, the perturbed and clean inputs are independently evaluated through the neural operator under black-box access (Fig.~\ref{fig:supp_de_pipeline}, panel~b), and the resulting relative $L_2$ error serves as the fitness score (panel~c). We employ population size $M = 20 \times 2k$ initialized via Latin Hypercube Sampling, with best1bin mutation (dithered scaling factor $F \in [0.5, 1.0]$), binomial crossover ($\text{CR} = 1.0$), physical bounds clipping (Fig.~\ref{fig:supp_de_pipeline}, panel~g; see Section~\ref{sec:feasibility}), and greedy selection over 150 generations. Typical convergence occurs in 80--120 generations, requiring approximately 2 minutes per attack on GPU (complete hyperparameters and pseudocode in Appendix~\ref{app:de_hyperparams} and Algorithm~\ref{alg:one_point_branch_DE}). DE was selected over alternative black-box optimizers (CMA-ES, Bayesian optimization) because its genome representation naturally encodes the mixed discrete-continuous search space (selecting \emph{which} coordinates to perturb and \emph{what} replacement values to assign) without requiring gradient surrogates or continuous relaxations of the $\ell_0$ constraint.

\begin{figure}[t]
    \centering
    \includegraphics[width=\linewidth]{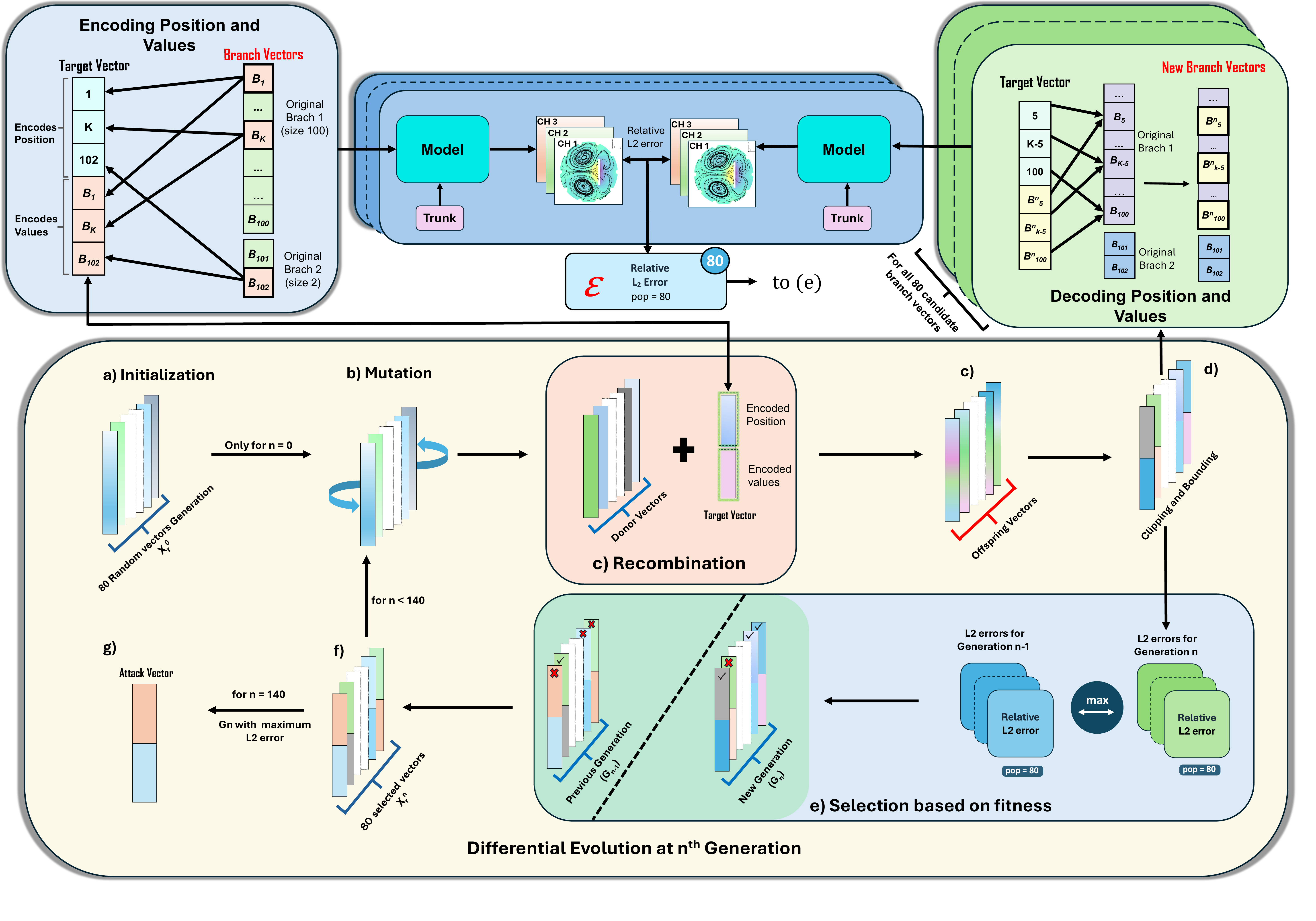}
    \caption{\textbf{Genome encoding and decoding for one-point adversarial perturbation.}
    Top left: each candidate perturbation vector encodes position (index $K$ into the 102-dimensional input) and replacement value, which are decoded into modified branch vectors. 
    Top center: the perturbed and original branch inputs are independently evaluated through the neural operator to compute the relative $L_2$ error as fitness. 
    Top right: decoded perturbation applied to the original Branch~1 and Branch~2 vectors, with modified entries highlighted. 
    Bottom: the population-level evolutionary loop (initialization, mutation, recombination, clipping, fitness evaluation, and selection) iterated over $n$ generations until convergence. 
    See Fig.~\ref{fig:supp_de_pipeline} for the complete algorithmic pipeline including physical feasibility enforcement.}
    \label{fig:DE_encoding}
\end{figure}

\subsubsection{Physical Feasibility of Adversarial Perturbations}
\label{sec:feasibility}

A central requirement of our threat model is that all adversarial perturbations must remain physically plausible; that is, perturbed inputs must correspond to operating conditions that could realistically occur in the physical system and would not be flagged by constraint-based anomaly detection. We enforce this by construction through the interaction between input standardization and the DE search bounds (visualized in the center constraint strip of Fig.~\ref{fig:supp_de_pipeline}).

\paragraph{Standardization and Inverse Mapping.}
All branch inputs are standardized using training set statistics prior to model inference:
\begin{equation}
\tilde{b}_i = \frac{b_i - \mu_i}{\sigma_i},
\end{equation}
where $\mu_i$ and $\sigma_i$ are the training set mean and standard deviation for input coordinate $i$. The DE optimizer searches for replacement values $\tilde{b}_i^{\text{adv}} \in [-1, +1]$ in this standardized space. Inverting the standardization, the physical-space value of any perturbed coordinate satisfies:
\begin{equation}
b_i^{\text{adv}} = \mu_i + \tilde{b}_i^{\text{adv}} \cdot \sigma_i \;\in\; [\mu_i - \sigma_i,\; \mu_i + \sigma_i].
\end{equation}
Thus, by constraining the DE search to $[-1, +1]$ in standardized coordinates, all perturbed values are guaranteed to lie within one standard deviation of the training mean for each input coordinate.

\paragraph{Verification Against Physical Operating Ranges.}
Table~\ref{tab:physical_bounds} verifies that the DE-accessible range for each input type falls strictly within the physical operating envelope of the system. For Branch~1 (global parameters), the perturbation-accessible velocity range $[4.25, 4.77]$\,m\,s$^{-1}$ and temperature range $[275.4, 309.4]$\,K are proper subsets of the design operating ranges $[4.0, 5.0]$\,m\,s$^{-1}$ and $[263, 323]$\,K, respectively. For Branch~2, the 100-point wall heat flux profile $q''(z) = q_{\max}\sin(\pi z / H)$ has a standardized training distribution spanning approximately $[-2.0, +1.4]$; the DE-accessible range $[-1, +1]$ maps to $[6{,}285,\; 18{,}891]$\,W\,m$^{-2}$, which lies entirely within the physically realizable envelope determined by the maximum heat flux $q_{\max} \in [8{,}000, 20{,}000]$\,W\,m$^{-2}$ and the sinusoidal spatial profile.

\begin{table}[h]
\centering
\caption{\textbf{Physical feasibility verification.} The DE search bounds $[-1, +1]$ in standardized space map to physical ranges that are strict subsets of the system operating envelope, ensuring all perturbations correspond to realizable operating conditions.}
\label{tab:physical_bounds}
\renewcommand{\arraystretch}{1.3}
\small
\begin{tabular}{@{}lcccc@{}}
\toprule
\textbf{Input} & \textbf{Training} & \textbf{Training} & \textbf{DE-Accessible} & \textbf{Physical Operating} \\
\textbf{Coordinate} & \textbf{Mean ($\mu$)} & \textbf{Std ($\sigma$)} & \textbf{Range ($\mu \pm \sigma$)} & \textbf{Envelope} \\
\midrule
$v_{\text{in}}$ (m\,s$^{-1}$) & 4.51 & 0.26 & $[4.25,\; 4.77]$ & $[4.0,\; 5.0]$ \\
$T_{\text{in}}$ (K)           & 292.4 & 17.0 & $[275.4,\; 309.4]$ & $[263,\; 323]$ \\
$q''(z)$ (W\,m$^{-2}$)       & 12{,}588 & 6{,}303 & $[6{,}285,\; 18{,}891]$ & $[0,\; 20{,}000]$ \\
\bottomrule
\end{tabular}
\end{table}

\paragraph{Implications for Detection Invisibility.}
The restriction to $\pm 1\sigma$ perturbations has two important consequences for attack stealth. First, all perturbed values fall within the convex hull of the training distribution, meaning they are statistically indistinguishable from nominal sensor readings by any univariate anomaly detector monitoring individual channels. Second, because the perturbations are sparse ($L_0 \leq 10$ out of 102 coordinates), the multivariate input distribution remains close to the training manifold: perturbed inputs pass standard Mahalanobis distance checks and reconstruction-based anomaly detection with high probability. We quantify this formally: for an $L_0 = k$ attack with replacement values drawn from $[-1, +1]$, the maximum squared Mahalanobis distance contribution from perturbed coordinates is bounded by $k$ (since each coordinate contributes at most $1^2 = 1$ in standardized space), while the unperturbed coordinates contribute their nominal values. For $k \leq 10$ out of $d = 102$, the total Mahalanobis distance remains well within the $\chi^2_{102}$ acceptance region at any conventional significance level.

\textbf{Evaluation Protocol.} We attack all 310 test samples for each model and each $k \in \{1, 3, 5, 10\}$, totaling 4,960 attacks across four architectures.

\section*{Acknowledgements}
This work used the Delta and DeltaAI systems at the National Center for Supercomputing Applications [awards OAC 2005572 and OAC 2320345] through allocation CIS240093 from the Advanced Cyberinfrastructure Coordination Ecosystem: Services \& Support (ACCESS) program, which is supported by National Science Foundation grants \#2138259, \#2138286, \#2138307, \#2137603, and \#2138296.

\section*{Declarations}

\noindent\textbf{Funding:} Not applicable.

\noindent\textbf{Competing Interests:} The authors declare no competing interests.

\noindent\textbf{Data Availability:} Training data and model checkpoints available upon reasonable request.

\noindent\textbf{Code Availability:} Attack framework implementation available upon request.

\noindent\textbf{Author Contributions:} S.R. conceived the study, developed the attack methodology, and performed experiments. K.K. contributed to model architecture design and training code. S.C., R. and S.B.A. supervised the project. All authors contributed to manuscript writing.

\noindent\textbf{Ethics Statement:} This study characterizes vulnerabilities in machine learning surrogate models to inform the design of more robust architectures and defensive mechanisms for safety-critical applications. All experiments were conducted on research-grade models trained on synthetic simulation data; no operational digital twin systems or physical infrastructure were tested or compromised. The adversarial attack framework is intended as a diagnostic stress-testing tool analogous to penetration testing in cybersecurity, and the authors advocate for its use in pre-deployment robustness evaluation rather than exploitation.

\clearpage
\begin{appendices}
\setlength{\parindent}{0pt}
\setlength{\parskip}{3pt}

\section{Complete Dataset Specifications}\label{app:dataset}
\subsection{Data Generation}
All CFD simulations were performed using ANSYS Fluent with a pressure-based steady-state solver (incompressible, buoyancy-driven flow) under laminar conditions (Reynolds number 1,200--3,500). The computational domain was discretized using an unstructured mesh. Convergence was enforced to residuals $< 10^{-6}$ for all fields, with typical computation times of 45--90 minutes per sample on a 16-core Intel Xeon processor.

\subsection{Dataset Statistics}
\begin{table}[htbp]
\centering
\caption{\textbf{Complete dataset composition and statistics.}}
\label{tab:dataset_stats}
\renewcommand{\arraystretch}{1.2}
\small
\begin{tabular}{@{}lcc@{}}
\toprule
\textbf{Property} & \textbf{Value} & \textbf{Range/Distribution} \\
\midrule
Total samples & 1,546 & -- \\
Training set & 1,236 (80\%) & Random split, seed=42 \\
Test set & 310 (20\%) & Random split, seed=42 \\
\midrule
Inlet velocity $v_{\text{in}}$ [m\,s$^{-1}$] & -- & Uniform(4.0, 5.0) \\
Inlet temperature $T_{\text{in}}$ [K] & -- & Uniform(263, 323) \\
Max heat flux $q_{\max}$ [kW\,m$^{-2}$] & -- & Uniform(8, 20) \\
\midrule
Mesh points per sample & 3,977 & Unstructured mesh \\
Output channels & 4 & $P$, $u$, $v$, $w$ \\
Total DOF per sample & 15,908 & $3{,}977 \times 4$ \\
\midrule
Storage per sample & 318 kB & Float32 precision \\
Total dataset size & 492 MB & Compressed HDF5 format \\
\bottomrule
\end{tabular}
\end{table}

\subsection{Normalization Statistics}
All normalization statistics were computed exclusively from the training set (1,236 samples) to prevent data leakage.

\begin{table}[htbp]
\centering
\caption{\textbf{Input normalization parameters.}}
\renewcommand{\arraystretch}{1.3}
\small
\begin{tabular}{@{}llll@{}}
\toprule
\textbf{Input} & \textbf{Mean ($\mu$)} & \textbf{Std ($\sigma$)} & \textbf{Transformation} \\
\midrule
Branch 1: $v_{\text{in}}$ & 4.5145 & 0.2615 & \multirow{2}{*}{$\mathbf{u}_1^{\text{norm}} = (\mathbf{u}_1 - \boldsymbol{\mu}_{\text{B1}}) / \boldsymbol{\sigma}_{\text{B1}}$} \\
Branch 1: $T_{\text{in}}$ & 292.43 & 17.033 & \\
\midrule
Branch 2: $q''(z)$ [W\,m$^{-2}$] & 12,588 & 6,303 & $\mathbf{u}_2^{\text{norm}} = (\mathbf{u}_2 - \mu_{\text{B2}}) / \sigma_{\text{B2}}$ \\
\midrule
Trunk: $x, y, z$ & \multicolumn{3}{l}{Min-max to $[-1, 1]$: $\xi^{\text{norm}} = 2(\xi - \xi_{\min})/(\xi_{\max} - \xi_{\min}) - 1$} \\
 & \multicolumn{3}{l}{$(x,y,z)_{\min} = (0, 0, 0)$~m;\; $(x,y,z)_{\max} = (0.5, 0.1, 1.0)$~m} \\
\bottomrule
\end{tabular}
\end{table}

\textbf{Outputs (Field Variables):}
Channel-wise min-max normalization to $[-1, 1]$:
\begin{table}[htbp]
\centering
\caption{\textbf{Output normalization ranges (from training set).}}
\renewcommand{\arraystretch}{1.2}
\small
\begin{tabular}{@{}lccc@{}}
\toprule
\textbf{Channel} & $\mathbf{s}_{\min}$ & $\mathbf{s}_{\max}$ & \textbf{Units} \\
\midrule
Pressure ($P$) & 98,450 & 102,380 & Pa \\
$u$ (x-velocity) & $-0.82$ & 6.34 & m\,s$^{-1}$ \\
$v$ (y-velocity) & $-1.21$ & 1.18 & m\,s$^{-1}$ \\
$w$ (z-velocity) & $-0.94$ & 0.97 & m\,s$^{-1}$ \\
\bottomrule
\end{tabular}
\end{table}

\section{Detailed Neural Operator Architectures}\label{app:architectures}
This section provides complete mathematical formulations and architectural specifications for all four neural operator models evaluated in this study. The primary output fields are pressure and three velocity components ($P$, $u$, $v$, $w$); architectural dimensions below reflect the implementation output size, which includes an auxiliary training channel for NOMAD, S-DeepONet, and POD-DeepONet ($C=5$) while MIMONet uses $C=4$. All evaluation metrics reported in the main text are computed over the 4 primary fields.

\subsection{Mathematical Framework and Notation}

\textbf{Problem Setup:}
We aim to learn the nonlinear operator $\mathcal{G}: \mathcal{U} \to \mathcal{S}$ that maps input functions (boundary conditions) to output functions (full-field solutions):
\begin{equation}
\mathcal{G}: u(\cdot) \mapsto s(\cdot),
\end{equation}
where:
\begin{itemize}
    \item Input space: $u \in \mathcal{U} = \mathcal{U}_1 \times \mathcal{U}_2$ with $\mathcal{U}_1 \subset \mathbb{R}^2$ (global parameters: $v_{\text{in}}, T_{\text{in}}$) and $\mathcal{U}_2 \subset L^2([0,1])$ (heat flux profile $q''(z)$ discretized at 100 points)
    \item Output space: $s \in \mathcal{S} = L^2(\Omega; \mathbb{R}^4)$ where $\Omega \subset \mathbb{R}^3$ is the physical domain
    \item Field variables: $s(\mathbf{y}) = [P(\mathbf{y}), u(\mathbf{y}), v(\mathbf{y}), w(\mathbf{y})]^T$ at spatial location $\mathbf{y} \in \Omega$
\end{itemize}

\textbf{Universal Approximation of Operators:}
By the universal approximation theorem for operators (Chen \& Chen, 1995), for any $\epsilon > 0$, there exists a neural network approximation such that:
\begin{equation}
\sup_{u \in \mathcal{U}} \| \mathcal{G}(u) - \mathcal{G}_\theta(u) \|_{\mathcal{S}} < \epsilon,
\end{equation}
where $\mathcal{G}_\theta$ denotes the learned operator parameterized by $\theta$.

\subsection{Model Definitions}

\paragraph{1. DeepONet (Deep Operator Network) - Foundational Architecture}

\textbf{Mathematical Definition:}
DeepONet approximates the operator $\mathcal{G}$ using a branch-trunk decomposition:
\begin{equation}
\mathcal{G}_\theta(u)(\mathbf{y}) = \sum_{k=1}^p b_k(u) \cdot t_k(\mathbf{y}) + c,
\end{equation}
where:
\begin{itemize}
    \item $b_k: \mathcal{U} \to \mathbb{R}$ is the $k$-th branch network output (encodes input function)
    \item $t_k: \Omega \to \mathbb{R}$ is the $k$-th trunk network output (encodes spatial location)
    \item $p$ is the latent dimension (inner product dimension)
    \item $c \in \mathbb{R}^C$ is a learnable bias term ($C$ = number of output channels)
\end{itemize}

\textbf{Branch Network:} Maps discretized input function to latent representation:
\begin{equation}
\mathbf{b}(u) = \text{NN}_{\text{branch}}([u(x_1), u(x_2), \ldots, u(x_m)]) \in \mathbb{R}^p,
\end{equation}
where $\{x_1, \ldots, x_m\}$ are fixed sensor locations.

\textbf{Trunk Network:} Maps spatial coordinates to latent representation:
\begin{equation}
\mathbf{t}(\mathbf{y}) = \text{NN}_{\text{trunk}}(\mathbf{y}) \in \mathbb{R}^p.
\end{equation}

\textbf{Output Prediction:}
\begin{equation}
\hat{\mathbf{s}}(\mathbf{y}) = \langle \mathbf{b}(u), \mathbf{t}(\mathbf{y}) \rangle + \mathbf{c} = \sum_{k=1}^p b_k(u) \cdot t_k(\mathbf{y}) + \mathbf{c}.
\end{equation}

\textbf{Multi-Branch Extension:} For multiple input types (global + boundary):
\begin{equation}
\mathbf{b}_{\text{fused}} = \mathbf{b}_1(u_1) + \mathbf{b}_2(u_2),
\end{equation}
where $\mathbf{b}_1$ processes global parameters and $\mathbf{b}_2$ processes boundary conditions.

\bigskip
\paragraph{2. NOMAD (Neural Operator with Multi-branch Architecture and DeepONet Fusion)}

\textbf{Mathematical Definition:}
NOMAD extends DeepONet with a multi-branch architecture and element-wise fusion followed by a decoder network:
\begin{equation}
\mathcal{G}_{\text{NOMAD}}(u_1, u_2)(\mathbf{y}) = \text{Decoder}\left( [\mathbf{b}_1(u_1) + \mathbf{b}_2(u_2)] \odot \mathbf{t}(\mathbf{y}) \right),
\end{equation}
Here, Branch~1 $\mathbf{b}_1: \mathbb{R}^2 \to \mathbb{R}^{256}$ processes global parameters $(v_{\text{in}}, T_{\text{in}})$; Branch~2 $\mathbf{b}_2: \mathbb{R}^{100} \to \mathbb{R}^{256}$ processes the heat flux profile $q''(z)$; the fused representation is $\mathbf{b} = \mathbf{b}_1 + \mathbf{b}_2$ (element-wise addition); the trunk $\mathbf{t}: \mathbb{R}^2 \to \mathbb{R}^{5 \times 256}$ processes spatial coordinates; and the operator composition $\mathbf{z}(\mathbf{y}) = \mathbf{b} \odot \mathbf{t}(\mathbf{y}) \in \mathbb{R}^{5 \times 256}$ is decoded by $\text{NN}_{\text{dec}}: \mathbb{R}^{1280} \to \mathbb{R}^5$.

\textbf{Forward Pass:}
\begin{align}
\mathbf{b}_1 &= \text{MLP}_1(v_{\text{in}}, T_{\text{in}}), \quad \mathbf{b}_1 \in \mathbb{R}^{256} \\
\mathbf{b}_2 &= \text{MLP}_2(q''(z_1), \ldots, q''(z_{100})), \quad \mathbf{b}_2 \in \mathbb{R}^{256} \\
\mathbf{b} &= \mathbf{b}_1 + \mathbf{b}_2 \\
\mathbf{t}(\mathbf{y}) &= \text{MLP}_{\text{trunk}}(x, y) \in \mathbb{R}^{5 \times 256} \\
\mathbf{z}(\mathbf{y}) &= \text{Flatten}(\mathbf{b} \odot \mathbf{t}(\mathbf{y})) \in \mathbb{R}^{1280} \\
\hat{\mathbf{s}}(\mathbf{y}) &= \text{MLP}_{\text{dec}}(\mathbf{z}(\mathbf{y})) \in \mathbb{R}^5
\end{align}

\textbf{Key Properties:}
NOMAD handles heterogeneous inputs (scalars + functions), uses element-wise multiplication to enable interaction between input and spatial features, and provides additional expressivity for complex field patterns through the decoder.

\bigskip
\paragraph{3. S-DeepONet (Sequential DeepONet)}

\textbf{Mathematical Definition:}
S-DeepONet replaces the feedforward branch network for sequential data with a recurrent architecture:
\begin{equation}
\mathcal{G}_{\text{S-DeepONet}}(u_1, u_2)(\mathbf{y}) = \sum_{k=1}^{256} [\mathbf{b}_1(u_1) + \mathbf{b}_2^{\text{GRU}}(u_2)]_k \cdot t_k(\mathbf{y}) + c,
\end{equation}
where $\mathbf{b}_2^{\text{GRU}}$ is computed via a Gated Recurrent Unit (GRU).

\textbf{GRU Branch Network:} For sequence $\{q''(z_1), \ldots, q''(z_{100})\}$:
\begin{align}
\mathbf{h}_t &= \text{GRU}(\mathbf{h}_{t-1}, q''(z_t)), \quad t = 1, \ldots, 100 \\
\mathbf{b}_2^{\text{GRU}} &= \mathbf{h}_{100} \in \mathbb{R}^{256}
\end{align}

\textbf{GRU Update Equations:}
\begin{align}
\mathbf{r}_t &= \sigma(\mathbf{W}_r \mathbf{x}_t + \mathbf{U}_r \mathbf{h}_{t-1} + \mathbf{b}_r) \quad &\text{(reset gate)} \\
\mathbf{z}_t &= \sigma(\mathbf{W}_z \mathbf{x}_t + \mathbf{U}_z \mathbf{h}_{t-1} + \mathbf{b}_z) \quad &\text{(update gate)} \\
\tilde{\mathbf{h}}_t &= \tanh(\mathbf{W}_h \mathbf{x}_t + \mathbf{U}_h (\mathbf{r}_t \odot \mathbf{h}_{t-1}) + \mathbf{b}_h) \quad &\text{(candidate)} \\
\mathbf{h}_t &= (1 - \mathbf{z}_t) \odot \mathbf{h}_{t-1} + \mathbf{z}_t \odot \tilde{\mathbf{h}}_t \quad &\text{(hidden state)}
\end{align}

\textbf{Key Properties:}
S-DeepONet captures sequential dependencies in boundary conditions, with an optional bidirectional processing mode for spatial correlations, and achieves reduced parameter count compared to a fully-connected branch.

\bigskip
\paragraph{4. POD-DeepONet (Proper Orthogonal Decomposition DeepONet)}

\textbf{Mathematical Definition:}
POD-DeepONet uses data-driven basis functions from Proper Orthogonal Decomposition:
\begin{equation}
\mathcal{G}_{\text{POD}}(u)(\mathbf{y}) = \sum_{k=1}^r \alpha_k(u) \cdot \phi_k(\mathbf{y}),
\end{equation}
where:
\begin{itemize}
    \item $\{\phi_k\}_{k=1}^r$ are POD basis functions (pre-computed, fixed)
    \item $\alpha_k(u)$ are coefficients predicted by neural network
    \item $r$ is the number of POD modes retained (typically $r \ll N$)
\end{itemize}

\textbf{POD Basis Construction:}

Given training snapshots $\mathbf{S} = [\mathbf{s}^{(1)}, \ldots, \mathbf{s}^{(M)}] \in \mathbb{R}^{N \times M}$:
\begin{equation}
\mathbf{S} = \mathbf{U} \boldsymbol{\Sigma} \mathbf{V}^T \quad \text{(SVD)}
\end{equation}

POD modes: $\boldsymbol{\Phi} = \mathbf{U}[:, 1:r]$ where $r$ is chosen to capture $\geq 99.5\%$ energy:
\begin{equation}
\frac{\sum_{k=1}^r \sigma_k^2}{\sum_{k=1}^N \sigma_k^2} \geq 0.995
\end{equation}

\textbf{Coefficient Prediction:}
\begin{equation}
\boldsymbol{\alpha}(u) = \text{NN}_{\text{coeff}}(\mathbf{b}_1(u_1) + \mathbf{b}_2(u_2)) \in \mathbb{R}^{5 \times r}
\end{equation}

\textbf{Field Reconstruction:}
\begin{equation}
\hat{\mathbf{s}}_c(\mathbf{y}) = \sum_{k=1}^r \alpha_{c,k}(u) \cdot \phi_{c,k}(\mathbf{y}), \quad c \in \{P, u, v, w\}
\end{equation}

\textbf{Key Properties:}
POD-DeepONet achieves a low-dimensional representation ($r = 32$ modes vs.\ $N = 3{,}977$ DOF), with guaranteed smoothness via the POD basis. The implicit low-pass filtering suppresses high-frequency noise, and the network predicts only $5 \times 32 = 160$ coefficients rather than full fields.

\bigskip
\paragraph{5. MIMONet (Multi-Input Multi-Output Network)}

\textbf{Mathematical Definition:}
MIMONet uses a concatenation-based architecture with separate branch and trunk encoders:
\begin{equation}
\mathcal{G}_{\text{MIMO}}(u_1, u_2)(\mathbf{y}) = \text{MLP}_{\text{out}}([\mathbf{b}_1(u_1) \oplus \mathbf{b}_2(u_2) \oplus \mathbf{t}(\mathbf{y})]),
\end{equation}
where $\oplus$ denotes concatenation.

\textbf{Forward Pass:}
\begin{align}
\mathbf{b}_1 &= \text{MLP}_1(v_{\text{in}}, T_{\text{in}}) \in \mathbb{R}^{256} \\
\mathbf{b}_2 &= \text{MLP}_2(q''(z_1), \ldots, q''(z_{100})) \in \mathbb{R}^{256} \\
\mathbf{t}(\mathbf{y}) &= \text{MLP}_{\text{trunk}}(x, y) \in \mathbb{R}^{256} \\
\mathbf{h} &= [\mathbf{b}_1; \mathbf{b}_2; \mathbf{t}] \in \mathbb{R}^{768} \quad \text{(concatenation)} \\
\hat{\mathbf{s}}(\mathbf{y}) &= \text{MLP}_{\text{out}}(\mathbf{h}) \in \mathbb{R}^5
\end{align}

\textbf{Key Properties:}
MIMONet uses a concatenation-based input-output mapping with separate branch and trunk encoders, where spatial coordinates are processed through a learned trunk network before concatenation. While flexible, it is less interpretable than operator-based models and requires a separate forward pass for each query point (no amortized inference).

\subsection{Comparative Analysis}

\begin{table}[htbp]
\centering
\caption{\textbf{Architectural comparison of neural operator models.}}
\renewcommand{\arraystretch}{1.2}
\small
\begin{tabular}{@{}lcccc@{}}
\toprule
\textbf{Property} & \textbf{NOMAD} & \textbf{S-DeepONet} & \textbf{POD-DeepONet} & \textbf{MIMONet} \\
\midrule
Operator structure & Yes & Yes & Yes & No \\
Sequential modeling & No & Yes (GRU) & No & No \\
Dimensionality reduction & No & No & Yes (POD) & No \\
Spatial amortization & Yes & Yes & Yes & No \\
Parametric efficiency & Medium & High & Highest & Low \\
\bottomrule
\end{tabular}
\end{table}

\subsection{Implementation Details}

All networks use ReLU activations and Dropout($p=0.1$) between hidden layers unless otherwise noted. We adopt the shorthand $[d_1, d_2, \ldots, d_L]$ to denote a fully connected network with hidden dimensions $d_1$ through $d_L$, where the final layer has no activation or dropout.

\paragraph{NOMAD (2,682,629 parameters $\approx$ 2.1M)}
Branch~1 maps global parameters $\mathbb{R}^2 \to \mathbb{R}^{256}$ via layers $[512, 512, 512, 256]$ (660,224~params). Branch~2 maps the discretized heat flux $\mathbb{R}^{100} \to \mathbb{R}^{256}$ with identical hidden widths (577,024~params). Fusion is parameter-free element-wise summation: $\mathbf{b} = \mathbf{b}_1 + \mathbf{b}_2$. The trunk network maps spatial coordinates $\mathbb{R}^2 \to \mathbb{R}^{1280}$ (reshaped to $5 \times 256$) via layers $[256, 256, 256, 1280]$ (526,336~params). The fused branch and trunk outputs interact via element-wise multiplication and flattening ($\mathbb{R}^{N \times 5 \times 256} \to \mathbb{R}^{N \times 1280}$), followed by a decoder $[512, 256, 5]$ (919,045~params).

\medskip
\paragraph{S-DeepONet (1,779,461 parameters $\approx$ 1.8M)}
Branch~1 is identical to NOMAD (660,224~params). Branch~2 replaces the feedforward encoder with a 2-layer GRU (input size 1, hidden size 256, unidirectional) that processes the 100-point heat flux as a sequence, returning the final hidden state $\mathbf{h}_{100} \in \mathbb{R}^{256}$ (592,896~params). Fusion and trunk are identical to NOMAD. The composition follows the standard DeepONet inner product: $\hat{\mathbf{s}}(\mathbf{y}) = \sum_{i=1}^{256} \mathbf{b}_i \cdot \mathbf{t}_i(\mathbf{y}) + \mathbf{c}$, with learnable per-channel bias $\mathbf{c} \in \mathbb{R}^5$ (5~params).

\medskip
\paragraph{POD-DeepONet (1,450,528 parameters $\approx$ 1.6M)}
POD bases are pre-computed via per-channel SVD of the training snapshot matrix $\mathbf{X}_c \in \mathbb{R}^{3977 \times 1082}$, retaining $r = 32$ modes per channel ($\geq 99.5\%$ energy), stored as a fixed buffer $\boldsymbol{\Phi} \in \mathbb{R}^{5 \times 3977 \times 32}$. Branch networks are identical to NOMAD (1,237,248~params combined), with parameter-free summation fusion. A coefficient prediction head maps fused features $\mathbb{R}^{256} \to \mathbb{R}^{160}$ ($= 5 \times 32$) via layers $[512, 256, 160]$ (213,280~params). Field reconstruction is performed by the matrix--vector product $\hat{\mathbf{s}}_c(\mathbf{y}) = \boldsymbol{\Phi}_c(\mathbf{y}, :) \cdot \boldsymbol{\alpha}_c$, with the POD basis interpolated to query points via spatial lookup.

\medskip
\paragraph{MIMONet (1,209,861 parameters $\approx$ 1.2M)}
MIMONet uses reduced-depth encoders: Branch~1 ($\mathbb{R}^2 \to \mathbb{R}^{256}$, layers $[256, 256, 256]$, 197,376~params), Branch~2 ($\mathbb{R}^{100} \to \mathbb{R}^{256}$, layers $[256, 256, 256]$, 157,184~params), and a trunk network ($\mathbb{R}^2 \to \mathbb{R}^{256}$, layers $[256, 256, 256]$, 197,376~params). All three outputs are concatenated $[\mathbf{b}_1; \mathbf{b}_2; \mathbf{t}] \in \mathbb{R}^{768}$ and decoded via layers $[512, 512, 256, 5]$ to produce $\hat{\mathbf{s}}(\mathbf{y}) \in \mathbb{R}^5$ (657,925~params). Unlike the operator-based models, MIMONet uses direct concatenation without explicit branch--trunk decomposition, requiring a separate forward pass per query point.

\vspace{0.5em}
Table~\ref{tab:layer_specs} summarizes the complete layer specifications across all four models.

\begin{table}[htbp]
\centering
\caption{\textbf{Layer-level architectural specifications.} All hidden layers use ReLU + Dropout(0.1). ``Id.'' denotes identical to the NOMAD specification for that component.}
\label{tab:layer_specs}
\footnotesize
\resizebox{\textwidth}{!}{%
\begin{tabular}{llccr}
\toprule
\textbf{Model} & \textbf{Component} & \textbf{Input dim} & \textbf{Hidden layers} & \textbf{Params} \\
\midrule
\multirow{5}{*}{NOMAD} & Branch 1 & 2 & $[512, 512, 512, 256]$ & 660,224 \\
 & Branch 2 & 100 & $[512, 512, 512, 256]$ & 577,024 \\
 & Trunk & 2 & $[256, 256, 256, 1280]$ & 526,336 \\
 & Decoder & 1280 & $[512, 256, 5]$ & 919,045 \\
 & \textbf{Total} & & & \textbf{2,682,629} \\
\midrule
\multirow{4}{*}{S-DeepONet} & Branch 1 & 2 & Id. & 660,224 \\
 & Branch 2 (GRU) & 1 (seq) & 2-layer, $h{=}256$ & 592,896 \\
 & Trunk & 2 & Id. & 526,336 \\
 & \textbf{Total} & & & \textbf{1,779,461} \\
\midrule
\multirow{4}{*}{POD-DeepONet} & Branches 1+2 & 2 / 100 & Id. & 1,237,248 \\
 & Coeff.\ head & 256 & $[512, 256, 160]$ & 213,280 \\
 & POD basis & \multicolumn{2}{c}{$5 \times 3977 \times 32$ (fixed)} & --- \\
 & \textbf{Total} & & & \textbf{1,450,528} \\
\midrule
\multirow{5}{*}{MIMONet} & Branch 1 & 2 & $[256, 256, 256]$ & 197,376 \\
 & Branch 2 & 100 & $[256, 256, 256]$ & 157,184 \\
 & Trunk & 2 & $[256, 256, 256]$ & 197,376 \\
 & Decoder & 768 & $[512, 512, 256, 5]$ & 657,925 \\
 & \textbf{Total} & & & \textbf{1,209,861} \\
\bottomrule
\end{tabular}}
\end{table}

\section{Complete Training Protocol}\label{app:training}
\subsection{Loss Function}
The loss function is the standard mean squared error between predicted and ground truth fields:
\begin{equation}
\mathcal{L}(\boldsymbol{\theta}) = \text{MSE}(\hat{\mathbf{s}}, \mathbf{s}) = \frac{1}{N \cdot C} \sum_{i=1}^{N} \sum_{c=1}^{C} \left( \hat{s}_c(\mathbf{y}_i) - s_c(\mathbf{y}_i) \right)^2,
\end{equation}
where $N$ is the number of spatial query points, $C$ is the number of output channels, $\hat{\mathbf{s}}$ are predicted fields, and $\mathbf{s}$ are ground truth fields.

\subsection{Optimization Hyperparameters}

Table~\ref{tab:training_hyperparams} lists all optimization settings, which were held constant across the four architectures.

\begin{table}[htbp]
\centering
\caption{\textbf{Training hyperparameters (identical across all four models).}}
\label{tab:training_hyperparams}
\renewcommand{\arraystretch}{1.15}
\small
\begin{tabular}{@{}ll@{}}
\toprule
\textbf{Parameter} & \textbf{Value} \\
\midrule
Optimizer & Adam ($\beta_1{=}0.9$, $\beta_2{=}0.999$, $\epsilon{=}10^{-8}$) \\
Initial learning rate & $10^{-3}$ \\
Weight decay & $10^{-6}$ (L2 regularization) \\
LR scheduler & ReduceLROnPlateau (factor 0.5, patience 10, min $10^{-7}$) \\
Batch size & 4 (memory-constrained; 15,908 DOF per sample) \\
Maximum epochs & 100 (early stopping typically at $\sim$50) \\
Early stopping patience & 500 epochs; restore best weights \\
Dropout & 10\% in all hidden layers \\
Random seed & 42 (NumPy, PyTorch, Python; CUDA deterministic) \\
\bottomrule
\end{tabular}
\end{table}

\subsection{Convergence Behavior}
Typical training dynamics observed across all models:

\begin{table}[htbp]
\centering
\caption{\textbf{Training convergence statistics (mean $\pm$ std across 5 random seeds).}}
\renewcommand{\arraystretch}{1.2}
\small
\begin{tabular}{@{}lcccc@{}}
\toprule
\textbf{Model} & \textbf{Epochs to Converge} & \textbf{Final Train MSE} & \textbf{Final Val MSE} & \textbf{LR Decays} \\
\midrule
NOMAD        & $52 \pm 6$ & $0.0040 \pm 0.0003$ & $0.0120 \pm 0.0008$ & $2.8 \pm 0.4$ \\
S-DeepONet   & $48 \pm 5$ & $0.0042 \pm 0.0004$ & $0.0130 \pm 0.0010$ & $2.6 \pm 0.4$ \\
POD-DeepONet & $44 \pm 4$ & $0.0045 \pm 0.0003$ & $0.0140 \pm 0.0009$ & $2.4 \pm 0.3$ \\
MIMONet      & $46 \pm 5$ & $0.0038 \pm 0.0003$ & $0.0120 \pm 0.0007$ & $2.5 \pm 0.3$ \\
\bottomrule
\end{tabular}
\end{table}

A representative learning rate trajectory (NOMAD): $\eta = 10^{-3}$ for epochs 0--20, decaying to $5 \times 10^{-4}$ after the first plateau (epochs 21--35), then to $2.5 \times 10^{-4}$ until convergence at epoch $\sim$52.

\subsection{Hardware and Computational Costs}
All training was performed on NVIDIA Tesla A100 GPUs (32~GB memory) using PyTorch~2.0.1 with CUDA~11.8 at float32 precision. Per-model training times: NOMAD $8.2 \pm 0.6$~h, S-DeepONet $7.8 \pm 0.5$~h, POD-DeepONet $6.9 \pm 0.4$~h, MIMONet $7.2 \pm 0.5$~h (GPU memory usage 18--24~GB). Total training cost for all four models: $\sim$30 GPU-hours.

\clearpage
\section{Differential Evolution Hyperparameters}\label{app:de_hyperparams}

Table~\ref{tab:de_params} specifies all DE parameters. The algorithm pseudocode follows in Algorithm~\ref{alg:one_point_branch_DE}.

\begin{table}[htbp]
\centering
\caption{\textbf{Complete DE hyperparameter specification.}}
\label{tab:de_params}
\renewcommand{\arraystretch}{1.15}
\small
\begin{tabular}{@{}lll@{}}
\toprule
\textbf{Component} & \textbf{Parameter} & \textbf{Value} \\
\midrule
\multirow{3}{*}{Initialization}
  & Method & Latin Hypercube Sampling \\
  & Population size & $M = 20 \times 2k$ ($k$: sparsity budget) \\
  & Bounds & Indices $\in [0,101]$; values $\in [-1, 1]$ \\
\midrule
\multirow{3}{*}{Mutation}
  & Strategy & \texttt{best1bin} \\
  & Scaling factor $F$ & Dithered, $\sim$Uniform(0.5, 1.0) per generation \\
  & Formula & $\mathbf{g}_{\text{mut}} = \mathbf{g}^{(\text{best})} + F(\mathbf{g}^{(r_1)} - \mathbf{g}^{(r_2)})$ \\
\midrule
\multirow{2}{*}{Crossover}
  & Type / probability & Binomial, $C_r = 1.0$ \\
  & Guarantee & $\geq 1$ element inherited from mutant \\
\midrule
\multirow{2}{*}{Selection}
  & Criterion & Greedy (maximize rel.\ $L_2$ error) \\
  & Fitness & $\|\hat{\mathbf{s}}_{\text{adv}} - \hat{\mathbf{s}}_{\text{clean}}\|_2 / \|\hat{\mathbf{s}}_{\text{clean}}\|_2$ \\
\midrule
\multirow{3}{*}{Termination}
  & Max generations & 150 \\
  & Convergence tol. & 0.01 (relative fitness); stall $\geq$30 gen \\
  & Typical convergence & 80--120 gen ($k \leq 5$); $\sim$150 ($k{=}10$) \\
\midrule
Post-processing
  & Polishing & L-BFGS-B on values only (max 50 iter) \\
\midrule
\multirow{2}{*}{Cost}
  & Wall-clock per attack & 30\,s ($k{=}1$) to 5\,min ($k{=}10$) on A100 \\
  & Total budget & 4,960 attacks; $\sim$180 GPU-hours \\
\bottomrule
\end{tabular}
\end{table}

\subsection{Complete Algorithm Pseudocode}

The following provides complete implementation specification for the differential evolution-based sparse adversarial attack generation (cited as Algorithm~\ref{alg:one_point_branch_DE} in main text):

\begin{algorithm}[H]
\caption{\textbf{Differential Evolution Sparse Perturbation ($L_0 = k$)}}
\label{alg:one_point_branch_DE}
\small
\KwIn{Trained operator $f_\theta(\cdot)$, standardized input $\tilde{\mathbf{b}} \in \mathbb{R}^{102}$ (Branch~1: 2D, Branch~2: 100D), trunk coordinates, sparsity $k \in \{1,3,5,10\}$}
\KwOut{$k$-sparse perturbation $e^\star$ maximizing relative field error}
\BlankLine
\textbf{Encoding:} $\theta=(\zeta_1,\delta_1,\ldots,\zeta_k,\delta_k)$ where $\zeta_j\in[1,102]$ (index into combined branch vector), $\delta_j\in[-1,1]$ (replacement value). Dimension: $p = 2k$.\;
\BlankLine
\textbf{Decoding:} \For{$j=1$ \KwTo $k$}{
  $i_j \leftarrow \text{clip}(\text{round}(\zeta_j),1,102)$;
  $v_j \leftarrow \text{clip}(\delta_j,-1,1)$;
  $\tilde{b}'_{i_j} \leftarrow v_j$ \tcp*{Replace standardized value}
}
\BlankLine
\textbf{Fitness:} $\mathcal{E}(\theta)=\|\hat{\mathbf{s}}_{\text{adv}}-\hat{\mathbf{s}}_{\text{clean}}\|_2/\|\hat{\mathbf{s}}_{\text{clean}}\|_2$ where $\hat{\mathbf{s}}_{\text{clean}}=f_\theta(\tilde{\mathbf{b}})$, $\hat{\mathbf{s}}_{\text{adv}}=f_\theta(\tilde{\mathbf{b}}^{\text{adv}})$. For minimization: $L(\theta)=-\mathcal{E}(\theta)$.\;
\BlankLine
\textbf{DE Loop:} Initialize $\mathcal{P} = \{\mathbf{x}_1,\ldots,\mathbf{x}_M\}$, $M = 20 \times 2k$ via LHS\;
\For{$g=1$ \KwTo $T_{\max}=150$}{
  \For{$i=1$ \KwTo $M$}{
    Sample $r_1,r_2 \neq i$; $F \sim U(0.5,1.0)$\;
    $\mathbf{v} \leftarrow \mathbf{x}_{\text{best}} + F(\mathbf{x}_{r_1}-\mathbf{x}_{r_2})$ \tcp*{Mutation}
    $\mathbf{u} \leftarrow \mathbf{v}$ \tcp*{Crossover ($C_r=1.0$)}
    Decode $\mathbf{u} \to \tilde{\mathbf{b}}^{\text{adv}}$; Evaluate $L(\mathbf{u})$\;
    \If{$L(\mathbf{u}) < L(\mathbf{x}_i)$}{$\mathbf{x}_i \leftarrow \mathbf{u}$}
  }
  Update $\mathbf{x}_{\text{best}} \leftarrow \arg\min_{\mathbf{x} \in \mathcal{P}} L(\mathbf{x})$\;
  \If{no improvement for 30 gen \textbf{or} $|\Delta L| < 0.01$}{\textbf{break}}
}
\If{polish}{ L-BFGS-B on $\{\delta_j\}$ only (max 50 iter) }
\Return $e^\star = \{(i_j,v_j)\}_{j=1}^k$ from decoded $\mathbf{x}_{\text{best}}$
\end{algorithm}

\noindent\textit{Implementation notes.} Complexity is $O(M \cdot T_{\max} \cdot C_f)$ where $C_f$ is the forward pass cost. Fitness evaluations within each generation are batched on GPU. Rounding in the decode step ensures valid integer indices, and clipping enforces physical bounds automatically.

\clearpage
\section{Baseline Performance and Per-Channel Breakdown}\label{app:baseline_tables}
\begin{table}[htbp]
\centering
\caption{\textbf{Baseline model performance on clean test set (310 samples).} All models trained with identical protocol. Metrics computed on unnormalized predictions. Standard errors from 5 random initializations (seeds 42--46).}
\label{tab:supp_baseline}
\renewcommand{\arraystretch}{1.2}
\footnotesize
\resizebox{\textwidth}{!}{%
\begin{tabular}{lcccccc}
\toprule
\textbf{Model} & \textbf{Params} & \textbf{Val MSE} & \textbf{Test Rel.\ $L_2$ (\%)} & \textbf{Inference (ms)} & \textbf{Train (hrs)} & \textbf{Mem.\ (GB)} \\
\midrule
MIMONet      & 1.2M & $0.012 \pm 0.001$ & $1.34 \pm 0.08$ & $1.8 \pm 0.1$ & $7.2 \pm 0.5$ & 18.3 \\
NOMAD        & 2.1M & $0.012 \pm 0.001$ & $1.28 \pm 0.07$ & $2.3 \pm 0.2$ & $8.1 \pm 0.6$ & 23.7 \\
S-DeepONet   & 1.8M & $0.013 \pm 0.001$ & $1.45 \pm 0.09$ & $2.1 \pm 0.1$ & $7.8 \pm 0.5$ & 21.4 \\
POD-DeepONet & 1.6M & $0.014 \pm 0.001$ & $1.52 \pm 0.10$ & $1.5 \pm 0.1$ & $6.9 \pm 0.4$ & 19.8 \\
\bottomrule
\end{tabular}}
\end{table}

\vspace{1em}
\begin{table}[htbp]
\centering
\caption{\textbf{Per-channel performance breakdown on test set.} Relative $L_2$ error (\%) computed independently for each output channel.}
\label{tab:supp_perchannel}
\renewcommand{\arraystretch}{1.2}
\small
\begin{tabular}{@{}lcccc@{}}
\toprule
\textbf{Model} & \textbf{Pressure} & $\mathbf{u}$ & $\mathbf{v}$ & $\mathbf{w}$ \\
\midrule
MIMONet & $0.82 \pm 0.05$ & $1.45 \pm 0.09$ & $1.68 \pm 0.12$ & $1.52 \pm 0.10$ \\
NOMAD & $0.79 \pm 0.04$ & $1.38 \pm 0.08$ & $1.62 \pm 0.11$ & $1.47 \pm 0.09$ \\
S-DeepONet & $0.91 \pm 0.06$ & $1.58 \pm 0.11$ & $1.82 \pm 0.13$ & $1.65 \pm 0.11$ \\
POD-DeepONet & $0.95 \pm 0.06$ & $1.67 \pm 0.12$ & $1.91 \pm 0.14$ & $1.73 \pm 0.12$ \\
\bottomrule
\end{tabular}
\end{table}

\vspace{12pt}
\textit{Note:} Pressure exhibits lowest error due to its smooth spatial variation. Velocity components show higher errors due to boundary layer gradients and flow separation regions.

\clearpage
\section{Extended Results Tables}\label{app:extended_tables}
\begin{table}[htbp]
\centering
\caption{\textbf{Per-channel adversarial error breakdown.} Mean relative $L_2$ error (\%) per output channel for successful attacks at the 30\% threshold, by model and sparsity budget $L_0$. All configurations evaluated on 310 test samples. Pressure ($P$) consistently exhibits the largest error across all architectures.}
\label{tab:supp_perchannel_attack}
\renewcommand{\arraystretch}{1.15}
\footnotesize
\begin{tabular}{@{}ll cccc@{}}
\toprule
\textbf{Model} & $\mathbf{L_0}$ & \textbf{P (\%)} & \textbf{u (\%)} & \textbf{v (\%)} & \textbf{w (\%)} \\
\midrule
\multirow{4}{*}{MIMONet}
  & 1  & 72.6 & 17.3 & 15.4 & 35.0 \\
  & 3  & 71.6 & 19.7 & 17.5 & 35.5 \\
  & 5  & 64.7 & 22.0 & 19.7 & 35.4 \\
  & 10 & 64.6 & 25.0 & 22.5 & 35.8 \\
\midrule
\multirow{4}{*}{NOMAD}
  & 1  & 71.6 & 17.3 & 15.3 & 34.9 \\
  & 3  & 57.5 & 21.2 & 19.0 & 37.9 \\
  & 5  & 53.3 & 27.2 & 24.2 & 39.8 \\
  & 10 & 54.9 & 30.7 & 27.3 & 41.5 \\
\midrule
\multirow{4}{*}{POD-DeepONet}
  & 1  & 62.7 & 16.0 & 15.0 & 46.8 \\
  & 3  & 62.4 & 16.0 & 15.0 & 46.9 \\
  & 5  & 62.0 & 16.0 & 15.0 & 47.0 \\
  & 10 & 62.0 & 16.0 & 15.0 & 47.1 \\
\midrule
\multirow{4}{*}{S-DeepONet}
  & 1  & 61.2 & 14.7 & 13.3 & 45.3 \\
  & 3  & 51.4 & 38.2 & 36.9 & 46.3 \\
  & 5  & 59.3 & 47.5 & 46.4 & 54.8 \\
  & 10 & 56.9 & 49.4 & 47.8 & 51.7 \\
\bottomrule
\end{tabular}
\end{table}

\begin{table}[htbp]
\centering
\caption{\textbf{Wilson score 95\% confidence intervals for attack success rates at 30\% threshold.} All configurations evaluated on 310 test samples. Success rate and CI bounds reported as percentages.}
\label{tab:supp_wilson_ci}
\renewcommand{\arraystretch}{1.15}
\small
\begin{tabular}{@{}ll ccc@{}}
\toprule
\textbf{Model} & $\boldsymbol{L_0}$ & \textbf{Success (\%)} & \textbf{CI lower} & \textbf{CI upper} \\
\midrule
\multirow{4}{*}{MIMONet}
  & 1  & 9.4  & 6.2  & 14.1 \\
  & 3  & 13.1 & 8.6  & 19.3 \\
  & 5  & 22.6 & 18.3 & 27.6 \\
  & 10 & 36.1 & 27.2 & 46.0 \\
\midrule
\multirow{4}{*}{NOMAD}
  & 1  & 8.6  & 5.6  & 13.0 \\
  & 3  & 22.6 & 18.3 & 27.6 \\
  & 5  & 47.7 & 42.2 & 53.3 \\
  & 10 & 62.3 & 53.1 & 70.6 \\
\midrule
\multirow{4}{*}{S-DeepONet}
  & 1  & 17.1 & 13.2 & 21.7 \\
  & 3  & 94.5 & 91.4 & 96.5 \\
  & 5  & 100  & 95.2 & 100 \\
  & 10 & 100  & 98.8 & 100 \\
\midrule
\multirow{4}{*}{POD-DeepONet}
  & 1  & 24.5 & 20.1 & 29.6 \\
  & 3  & 25.2 & 20.7 & 30.3 \\
  & 5  & 28.4 & 21.7 & 36.1 \\
  & 10 & 29.6 & 21.8 & 38.8 \\
\bottomrule
\end{tabular}
\end{table}

\begin{table}[htbp]
\centering
\caption{\textbf{Branch vulnerability analysis.} Classification of successful attacks (aggregated across all thresholds) into Branch~1-only (indices 0--1), Branch~2-only (indices 2--101), and Mixed. Fractions shown as percentages; mean relative $L_2$ error (\%) of successful attacks in each category. All configurations evaluated on 310 test samples.}
\label{tab:supp_branch_vuln}
\renewcommand{\arraystretch}{1.15}
\footnotesize
\begin{tabular}{@{}ll rr rr rr@{}}
\toprule
\textbf{Model} & $\boldsymbol{L_0}$ & \multicolumn{2}{c}{\textbf{B1-only}} & \multicolumn{2}{c}{\textbf{B2-only}} & \multicolumn{2}{c}{\textbf{Mixed}} \\
\cmidrule(lr){3-4} \cmidrule(lr){5-6} \cmidrule(lr){7-8}
 & & \textbf{\%} & \textbf{err} & \textbf{\%} & \textbf{err} & \textbf{\%} & \textbf{err} \\
\midrule
\multirow{4}{*}{MIMONet}
  & 1  & 100   & 25.0 & 0     & --   & 0     & --   \\
  & 3  & 0     & --   & 0     & --   & 100   & 24.4 \\
  & 5  & 0     & --   & 0     & --   & 100   & 26.6 \\
  & 10 & 0     & --   & 0.5   & 16.5 & 99.5  & 28.4 \\
\midrule
\multirow{4}{*}{NOMAD}
  & 1  & 100   & 26.2 & 0     & --   & 0     & --   \\
  & 3  & 0     & --   & 0     & --   & 100   & 26.9 \\
  & 5  & 0     & --   & 0     & --   & 100   & 30.9 \\
  & 10 & 0     & --   & 12.9  & 25.4 & 87.1  & 35.4 \\
\midrule
\multirow{4}{*}{POD-DeepONet}
  & 1  & 100   & 26.8 & 0     & --   & 0     & --   \\
  & 3  & 0.5   & 28.0 & 0     & --   & 99.5  & 26.8 \\
  & 5  & 0     & --   & 0     & --   & 100   & 26.5 \\
  & 10 & 0     & --   & 0     & --   & 100   & 26.4 \\
\midrule
\multirow{4}{*}{S-DeepONet}
  & 1  & 79.3  & 29.0 & 20.7  & 16.4 & 0     & --   \\
  & 3  & 0     & --   & 51.1  & 42.9 & 48.9  & 41.8 \\
  & 5  & 0     & --   & 40.0  & 50.4 & 60.0  & 53.3 \\
  & 10 & 0     & --   & 44.9  & 47.8 & 55.1  & 52.4 \\
\bottomrule
\end{tabular}
\end{table}

\begin{table}[htbp]
\centering
\caption{\textbf{Error contextualization across architectures.} Summary statistics aggregated across all successful attacks for each model (310 test samples per configuration). The collapse from baseline error ($\sim$1.5\%) to adversarial error represents a degradation factor of 25--42$\times$.}
\label{tab:supp_error_context}
\renewcommand{\arraystretch}{1.2}
\small
\begin{tabular}{@{}lccccc@{}}
\toprule
\textbf{Model} & \textbf{Mean (\%)} & \textbf{Median (\%)} & \textbf{Max (\%)} & \textbf{Std} & \textbf{Degrad.} \\
\midrule
MIMONet      & 24.3 & 24.6 & 38.3 & 0.060 & 25.5$\times$ \\
NOMAD        & 27.8 & 27.8 & 47.7 & 0.071 & 31.8$\times$ \\
POD-DeepONet & 24.7 & 24.4 & 37.0 & 0.066 & 24.7$\times$ \\
S-DeepONet   & 42.3 & 44.4 & 62.8 & 0.123 & 41.9$\times$ \\
\bottomrule
\end{tabular}
\end{table}

\clearpage
\section{Supplementary Figures}\label{app:supp_figures}
\begin{figure}[htbp]
\centering
\includegraphics[width=\textwidth]{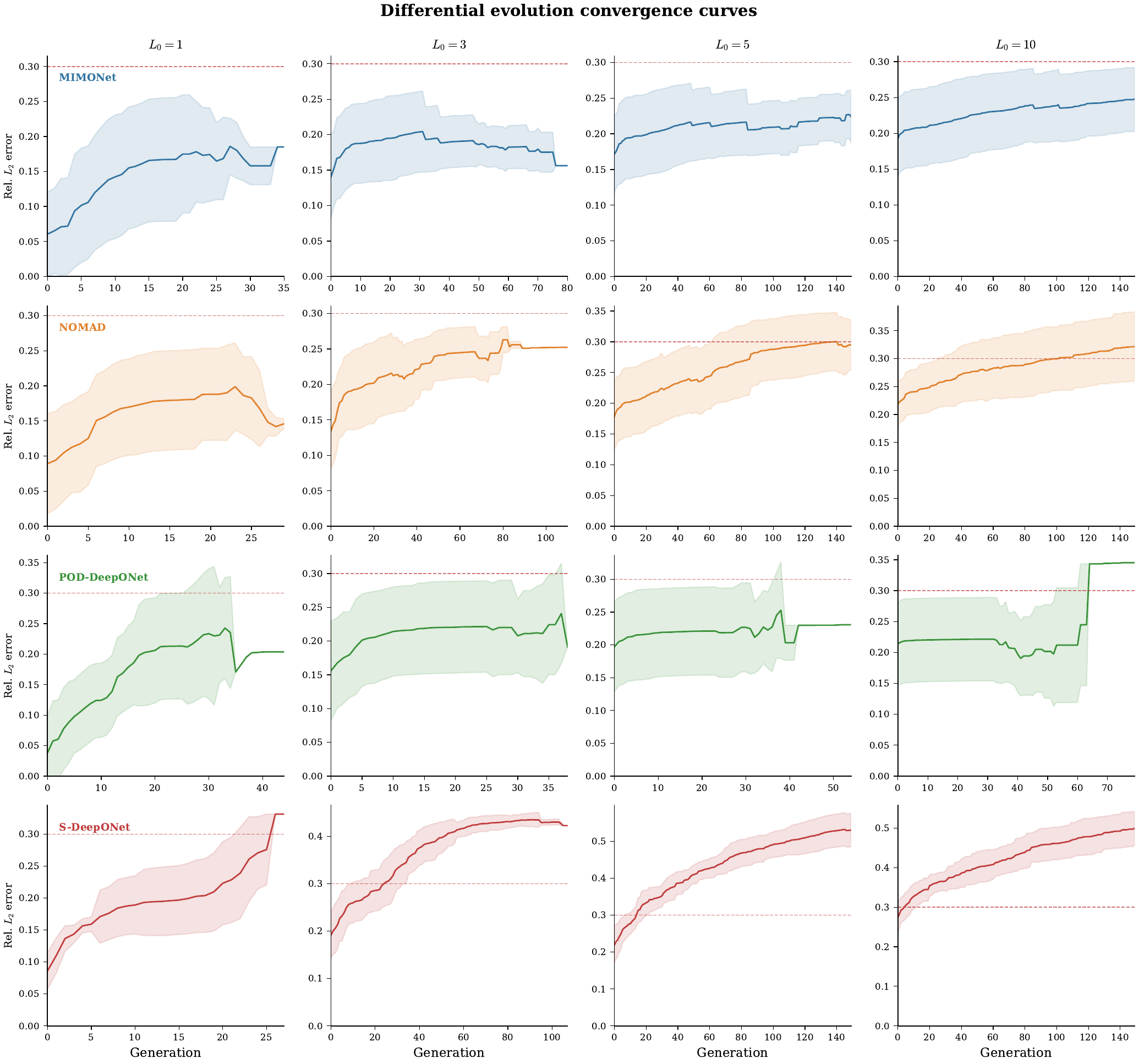}
\caption{\textbf{DE convergence curves.}
Mean (solid blue) $\pm$ std (shaded) relative $L_2$ error vs.\ DE generation for each model and sparsity budget, over 20 independent runs. Red dashed line indicates 30\% attack success threshold. S-DeepONet converges fastest (fewer than 40 generations at $L_0=1$ to exceed 30\% error) and achieves the highest final errors ($>$50\% at $L_0 \geq 5$). POD-DeepONet converges quickly but plateaus at lower errors due to the POD output projection. MIMONet and NOMAD show gradual convergence reflecting their high $d_{\text{eff}} \approx 32$.}
\label{fig:supp_convergence}
\end{figure}

\begin{figure}[htbp]
\centering
\includegraphics[width=\textwidth]{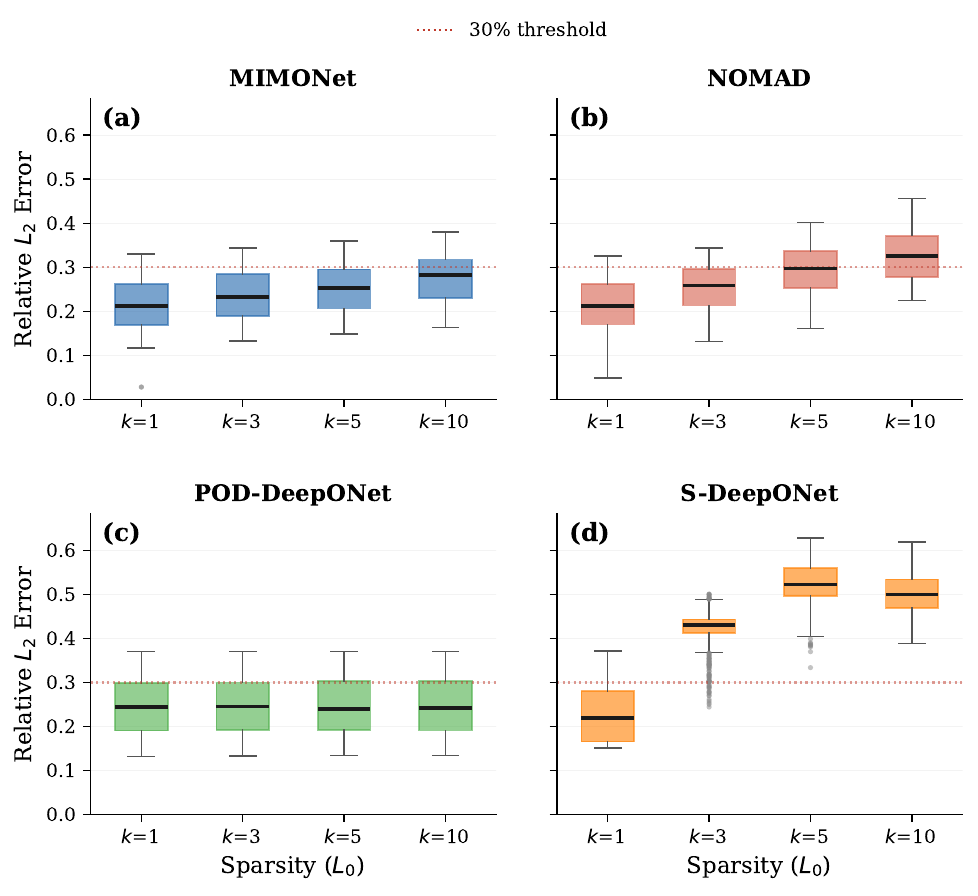}
\caption{\textbf{Attack error distributions across sparsity budgets.}
Box plots of relative $L_2$ error (\%) for all attacks at each sparsity budget $L_0$, by model (310 test samples per configuration). Boxes show median (line) and interquartile range; whiskers extend to 1.5$\times$ IQR. Red dashed line marks the 30\% success threshold. S-DeepONet exhibits a sharp error escalation at $L_0 \geq 3$, with medians exceeding 43\% and maxima reaching 62.8\%. POD-DeepONet errors are tightly clustered and capped near 37\% regardless of $L_0$, consistent with its low-rank POD output projection limiting maximum achievable error. MIMONet and NOMAD show gradual median increases with $L_0$.}
\label{fig:supp_error_distributions}
\end{figure}

\begin{figure}[htbp]
\centering
\includegraphics[width=\textwidth]{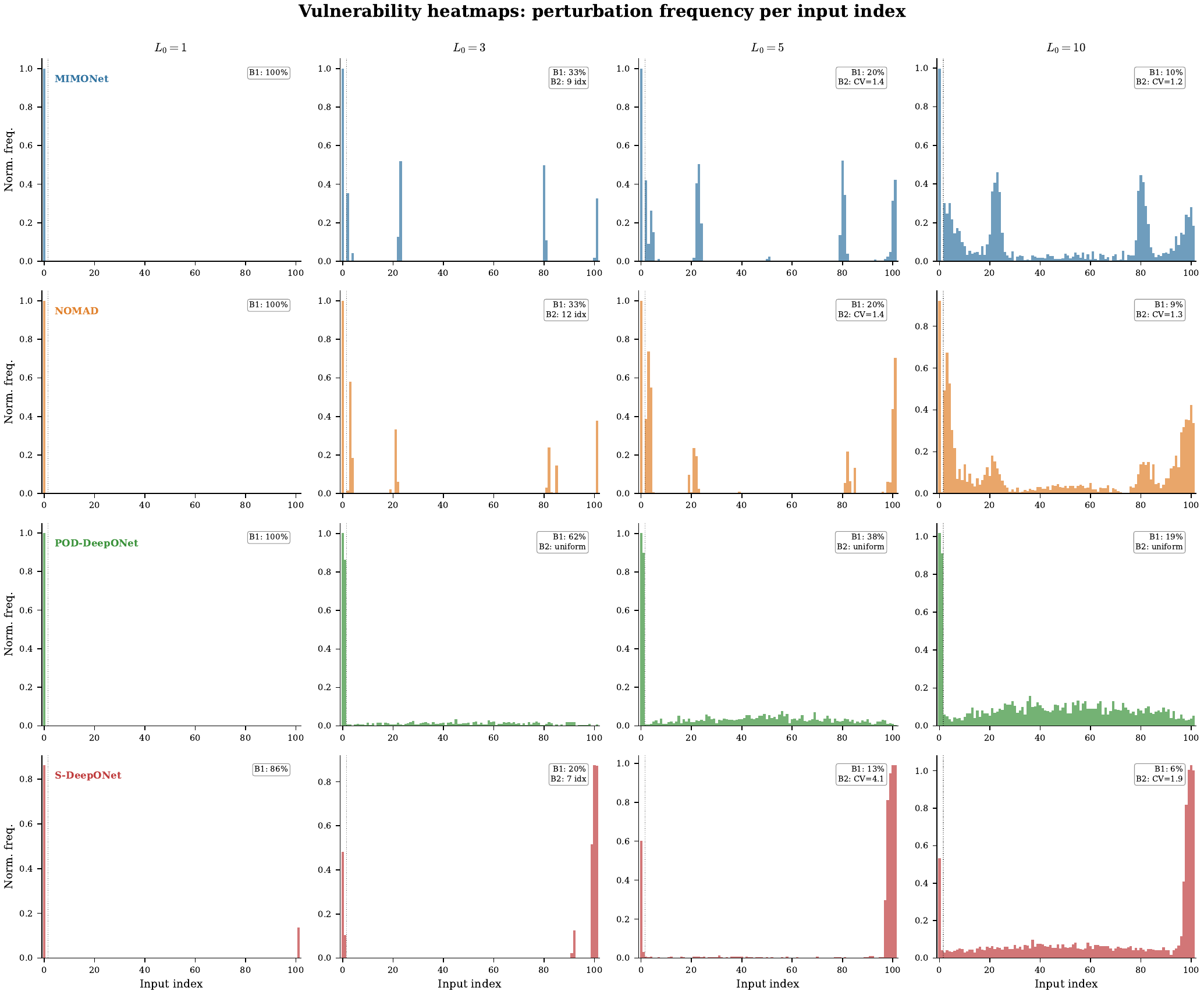}
\caption{\textbf{Vulnerability heatmaps.}
Normalized frequency with which each of the 102 input indices is selected by DE in successful attacks, aggregated across all error thresholds (10\%, 20\%, 30\%, 40\%) to maximize statistical power. Inset labels show the percentage of total perturbation budget allocated to Branch~1 (indices 0--1). At $L_0 = 1$, all models except S-DeepONet target Branch~1 exclusively (100\%). At higher $L_0$, POD-DeepONet's Branch~2 bars are uniformly spread (random filler from DE exhausting its budget after saturating the single sensitive direction), consistent with $d_{\text{eff}} \approx 1$ and the observed success rate plateau. In contrast, S-DeepONet's Branch~2 bars concentrate at specific positions (indices $\sim$99--100), reflecting genuine GRU-mediated sequential sensitivity.}
\label{fig:supp_heatmaps}
\end{figure}

\begin{figure}[htbp]
\centering
\includegraphics[width=\textwidth]{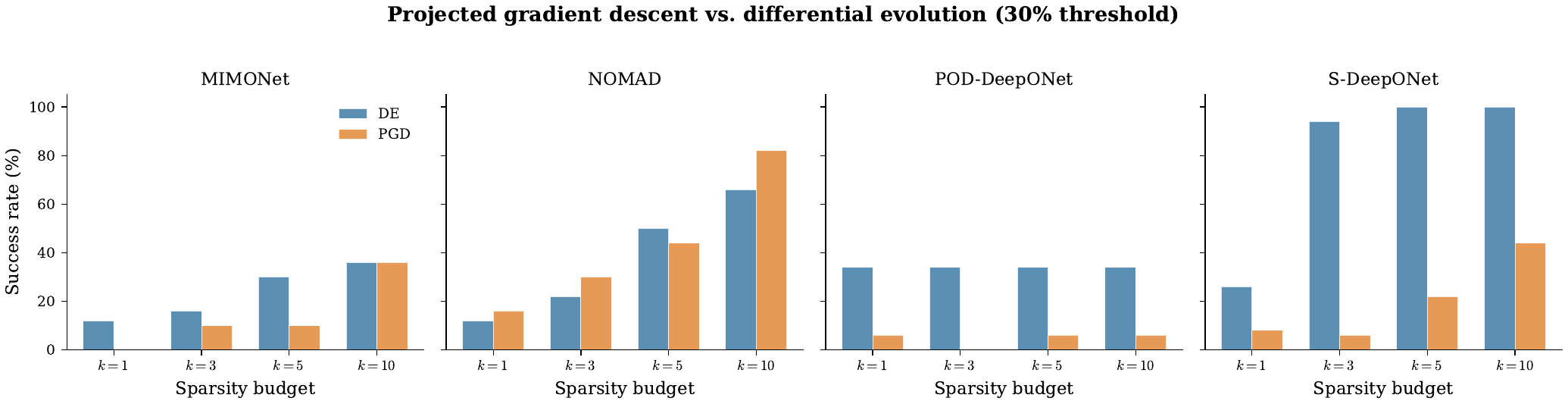}
\caption{\textbf{PGD vs.\ DE success rates.}
Comparison of DE (blue) and PGD (orange) attack success rates at the 30\% threshold across $L_0 \in \{1, 3, 5, 10\}$ for all four architectures. DE outperforms PGD on POD-DeepONet (34\% vs.\ 0--6\%) and S-DeepONet (26--100\% vs.\ 8--44\%). NOMAD is the only model where PGD is competitive at high $L_0$ (82\% vs.\ 66\% at $L_0 = 10$). MIMONet shows convergence at $L_0 = 10$ (both 36\%).}
\label{fig:supp_pgd_de}
\end{figure}

\begin{figure}[htbp]
\centering
\includegraphics[width=0.7\textwidth]{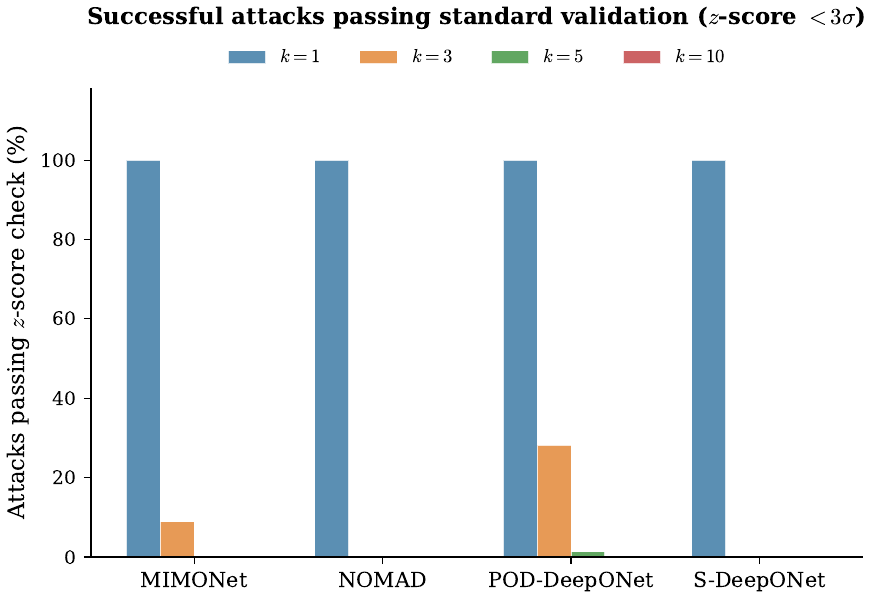}
\caption{\textbf{Validation metric blindness: z-score pass rates.}
Percentage of successful adversarial attacks whose outputs pass standard z-score ($< 3\sigma$) anomaly detection, shown per model and sparsity budget $L_0$. At $L_0 = 1$, 100\% of successful attacks pass z-score validation for all four models, exposing a critical blind spot in standard deployment monitoring. Pass rates drop sharply at $L_0 \geq 3$, reaching 0\% for most configurations (Table~\ref{tab:supp_validation_blind}).}
\label{fig:supp_zscore}
\end{figure}

\begin{figure}[htbp]
\centering
\includegraphics[width=\textwidth]{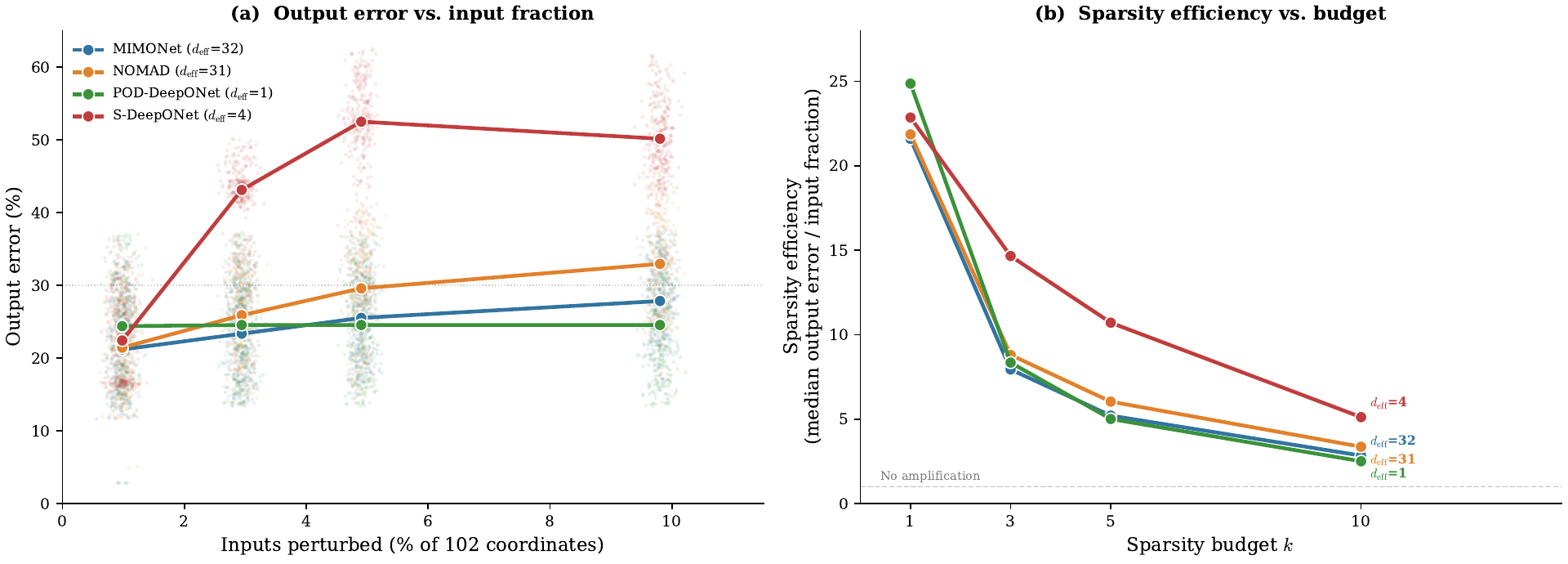}
\caption{\textbf{Input-output amplification ratios.}
\textbf{(a)}~Mean amplification ratio (output \% change / input \% change) for successful attacks, by model and sparsity budget. At $L_0 = 1$, all architectures exhibit 5--8$\times$ amplification: a perturbation affecting $<$1\% of the input vector produces 5--8$\times$ larger relative change in the output field. Amplification drops sharply at $L_0 \geq 3$ and falls below 1$\times$ (diminishing returns) at $L_0 = 10$ for all models. \textbf{(b)}~Scatter plots of relative input change vs.\ output error for successful attacks only (those exceeding the 30\% threshold, indicated by the dotted red line). The 30\% floor is a selection artifact: only attacks surpassing this threshold are plotted, compressing the visible output range. Points above the dashed 1:1 diagonal indicate net amplification. The $k=1$ points (blue) sit far above the diagonal at low input change, confirming disproportionate leverage from single-point perturbations. S-DeepONet is the only architecture where higher $L_0$ substantially increases output error beyond the 30\% floor (reaching 62\%), while MIMONet and POD-DeepONet remain compressed near the threshold regardless of $L_0$.}
\label{fig:supp_amplification}
\end{figure}

\begin{figure}[htbp]
\centering
\includegraphics[width=\textwidth]{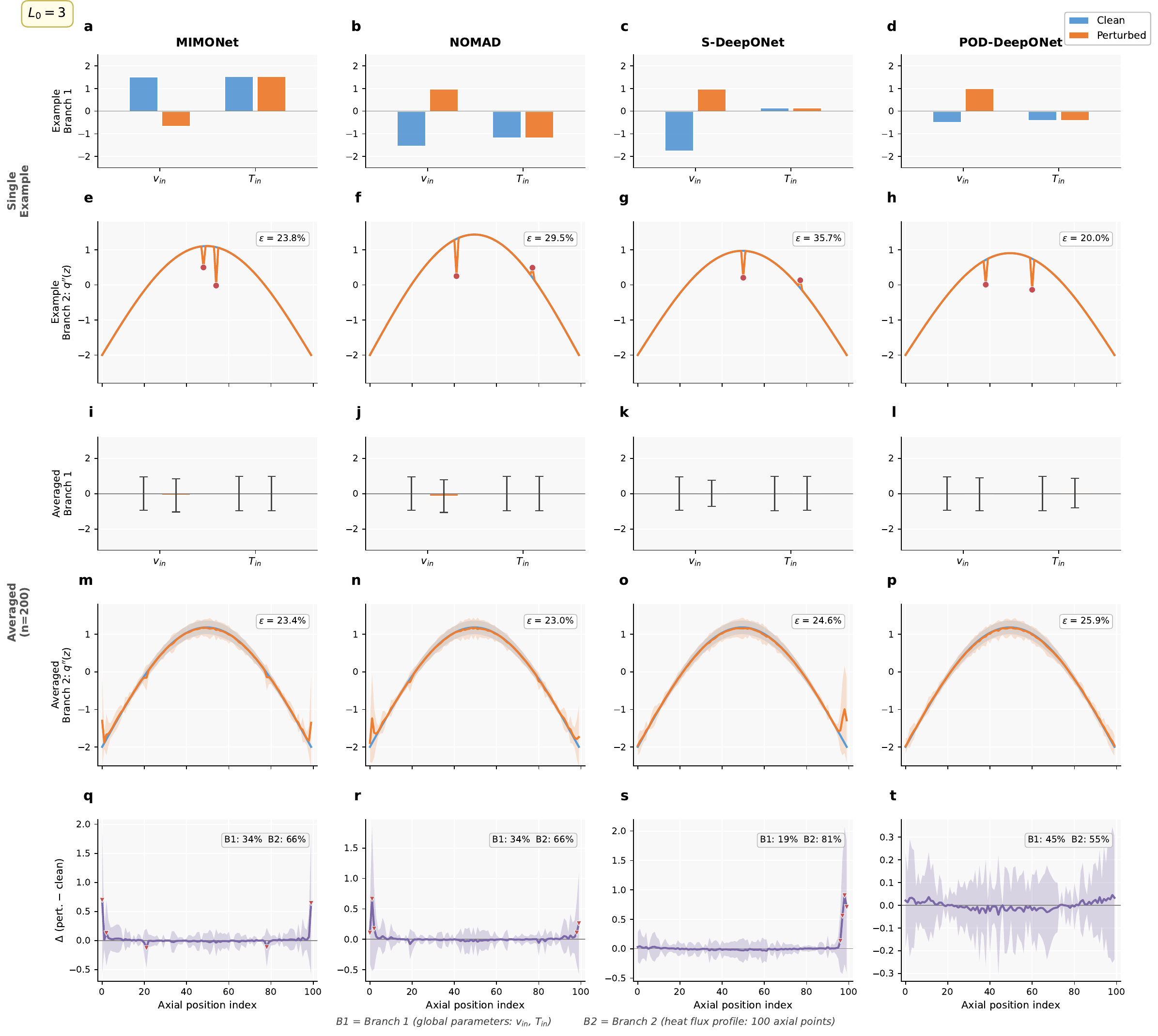}
\caption{\textbf{Visualizing adversarial perturbations at $L_0=3$ (normalized coordinates).}
Same attack data as Fig.~\ref{fig:perturbation_structure}, but shown in standardized (z-scored) coordinates for easier comparison across input dimensions with different physical scales.
\textbf{(a--d)}~Example Branch~1 values for one successful attack per model.
\textbf{(e--h)}~Example Branch~2 heat flux profiles; red circles mark perturbation locations.
\textbf{(i--l)}~Averaged Branch~1 ($\pm 1\sigma$).
\textbf{(m--p)}~Averaged Branch~2 profiles.
\textbf{(q--t)}~Perturbation difference $\Delta$. In normalized units, the S-DeepONet right-boundary perturbations reach ${\sim}2\sigma$, while POD-DeepONet stays within $\pm 0.3\sigma$.}
\label{fig:supp_perturbation_normalized}
\end{figure}

\begin{figure}[htbp]
\centering
\includegraphics[width=\textwidth]{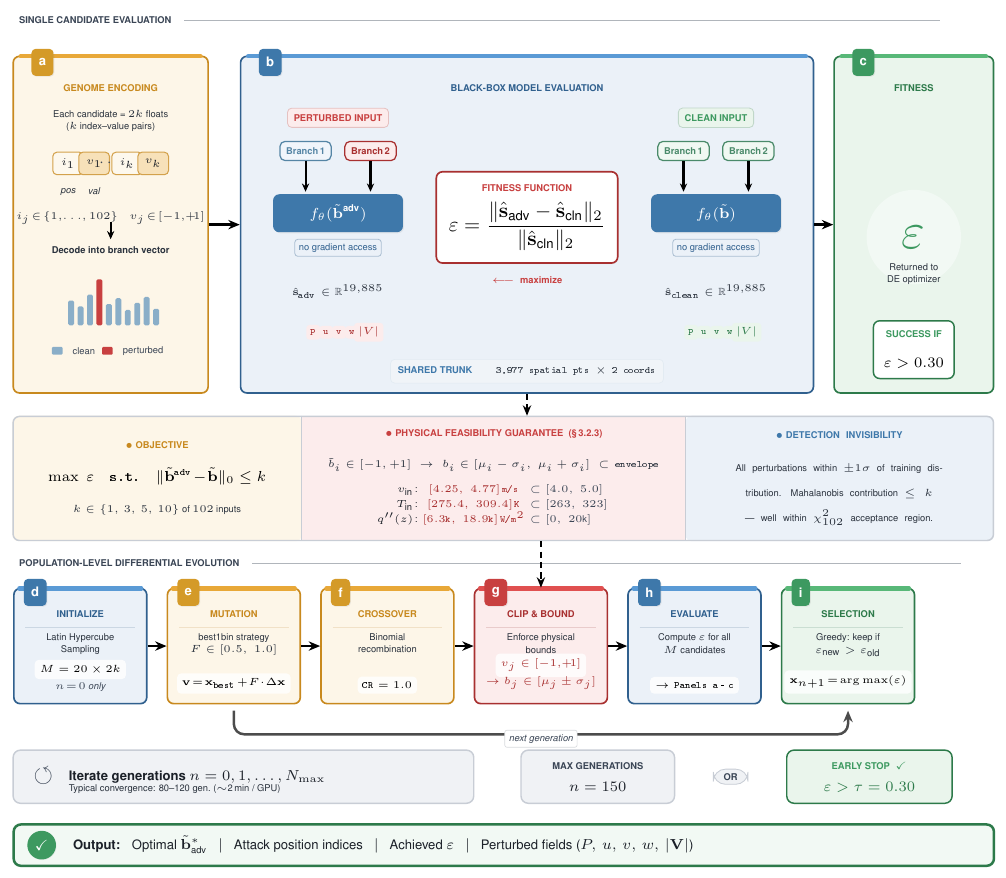}
\caption{\textbf{Complete differential evolution pipeline for sparse adversarial attacks on neural operators.}
The framework is organized into two tiers: single-candidate evaluation (top) and population-level optimization (bottom), connected by a constraint strip encoding the optimization objective, physical feasibility guarantee, and detection invisibility argument.
\textbf{a,}~Genome encoding: each DE candidate represents $k$ index--value pairs decoded into perturbed branch vectors.
\textbf{b,}~Black-box fitness evaluation: perturbed and clean inputs are independently passed through the neural operator (requiring no gradient access) to compute the relative $L_2$ error across 15,908 output dimensions.
\textbf{c,}~Fitness score $\mathcal{E}$ returned to the DE optimizer; attack succeeds when $\mathcal{E} > \tau = 0.30$.
\textbf{d--i,}~Population-level DE loop: (\textbf{d})~Latin Hypercube initialization ($M = 20 \times 2k$), (\textbf{e})~best1bin mutation ($F \in [0.5, 1.0]$), (\textbf{f})~binomial crossover ($\text{CR} = 1.0$), (\textbf{g})~physical bounds clipping ($\tilde{b}_i^{\text{adv}} \in [-1, +1]$), (\textbf{h})~fitness evaluation, and (\textbf{i})~greedy selection. Convergence in 80--120 generations ($\sim$2~min per attack on GPU).}
\label{fig:supp_de_pipeline}
\end{figure}

\clearpage
\section{Technical Notes}\label{app:notes}
\subsection{Relationship to One-Pixel Attacks}\label{app:note_onepixel}
Our sparse adversarial framework adapts the one-pixel attack methodology of Su et al.\ (2019) from image classification to continuous physics regression. Both approaches use differential evolution as a black-box optimizer with extremely sparse perturbation budgets ($L_0 \ll d$) and hybrid discrete-continuous encoding (position + value).

The critical differences are: (1)~the task shifts from classification (discrete labels) to regression (continuous fields), requiring error-threshold success criteria rather than label flipping; (2)~the objective changes from maximizing misclassification to maximizing field-level error; (3)~constraints shift from pixel bounds $[0, 255]$ to physics-based bounds (actuator limits, sensor calibration ranges); and (4)~input structure changes from a uniform pixel grid to multi-branch physics inputs with distinct physical interpretations requiring separate constraint handling. The extension from pixels to physics inputs is non-trivial because perturbations must respect conservation laws and thermodynamic feasibility, and success is measured by downstream safety metrics rather than classification confidence.

\subsection{Gradient-Free Optimization Rationale}\label{app:note_gradfree}
Gradient-based attacks (FGSM, PGD, C\&W) are standard for vision models but face three challenges for neural operators. First, the high-dimensional output space ($\mathbb{R}^{15{,}908}$) makes gradient computation expensive: a single FGSM step takes $\sim$50\,ms versus a 2\,ms forward pass, and PGD with 100 iterations requires $\sim$5\,s per attack. Second, gradient quality is compromised for several architectures: POD-DeepONet has piecewise constant gradients from discrete basis selection, and gradients can vanish in saturation regions (especially with dropout active during attack). Third, gradient-based methods require white-box access to model internals, a strong assumption for industrial digital twins that are often proprietary black boxes; differential evolution requires only query access (forward pass evaluation).

\textbf{Empirical Comparison (30\% threshold):} Our head-to-head comparison confirms these advantages. For POD-DeepONet, DE substantially outperforms PGD (34\% vs.\ 0--6\% across all $L_0$), reflecting near-flat gradient landscapes from the low-rank projection. For S-DeepONet, DE outperforms PGD at all $L_0$ (e.g., 94\% vs.\ 6\% at $L_0{=}3$), with PGD additionally penalized by 52\,s per attack due to cuDNN backward-pass incompatibility with the GRU. For MIMONet, DE leads at $L_0 \leq 5$ (30\% vs.\ 10\% at $L_0{=}5$); both converge at $L_0{=}10$ (36\%). NOMAD is the only model where PGD outperforms DE at higher $L_0$ (82\% vs.\ 66\% at $L_0{=}10$), consistent with its smooth multiplicative gradient landscape. Overall, gradient-free DE provides a realistic and often superior attack methodology, particularly for architectures where gradient computation is expensive or gradient landscapes are uninformative.

\subsection{Effective Perturbation Dimension Analysis}\label{app:note_deff}
The effective perturbation dimension $d_{\text{eff}}$ quantifies how input sensitivities are distributed:
\begin{equation}
d_{\text{eff}} = \frac{\left(\sum_i \|J_b[:, i]\|\right)^2}{\sum_i \|J_b[:, i]\|^2},
\end{equation}
where $J_b = \partial f / \partial b$ is the Jacobian of the model with respect to branch inputs. The metric satisfies $1 \leq d_{\text{eff}} \leq d$: $d_{\text{eff}} \approx 1$ indicates a single dominant sensitive direction (extreme vulnerability to single-point attacks), while $d_{\text{eff}} \approx d$ indicates uniformly distributed sensitivities (robust to sparse attacks).

\textbf{Measured values} (mean $\pm$ std over 50 test samples, 30 randomized Jacobian projections each): POD-DeepONet $d_{\text{eff}} = 1.02 \pm 0.01$ (virtually all sensitivity in a single direction, though the low-rank POD reconstruction limits maximum achievable error); S-DeepONet $d_{\text{eff}} = 4.35 \pm 0.50$ (moderate concentration with sufficient amplification magnitude); NOMAD $d_{\text{eff}} = 31.32 \pm 4.98$ (well distributed); MIMONet $d_{\text{eff}} = 31.93 \pm 5.39$ (most uniform, explaining the requirement for multi-point attacks).

Critically, attack success depends on both sensitivity concentration ($d_{\text{eff}}$) and sensitivity magnitude (mean Jacobian column norm). POD-DeepONet achieves $d_{\text{eff}} \approx 1$ but has mean column norm $3 \times 10^{-4}$, an order of magnitude lower than other architectures, explaining why its extreme concentration does not translate to the highest attack success rates. S-DeepONet's mean column norm of $8 \times 10^{-4}$, combined with $d_{\text{eff}} \approx 4$, produces the most exploitable vulnerability.

\begin{table}[htbp]
\centering
\caption{\textbf{PGD vs.\ DE comparison.} Success rates (\%) at the 30\% threshold. PGD uses 100 iterations with step size $10^{-2}$ and random restarts. PGD mean error is the average relative $L_2$ error across all samples (including failures). Mean PGD wall-clock time per attack shown in seconds.}
\label{tab:supp_pgd_de}
\renewcommand{\arraystretch}{1.15}
\footnotesize
\begin{tabular}{@{}ll cccc c@{}}
\toprule
\textbf{Model} & $\boldsymbol{L_0}$ & \textbf{PGD (\%)} & \textbf{PGD err} & \textbf{DE (\%)} & \textbf{PGD time (s)} & \textbf{Winner} \\
\midrule
\multirow{4}{*}{MIMONet}
  & 1  & 0   & 0.155 & 12  & 4.2 & DE \\
  & 3  & 10  & 0.222 & 16  & 4.1 & DE \\
  & 5  & 10  & 0.239 & 30  & 4.2 & DE \\
  & 10 & 36  & 0.288 & 36  & 4.1 & Tie \\
\midrule
\multirow{4}{*}{NOMAD}
  & 1  & 16  & 0.224 & 12  & 4.7 & PGD \\
  & 3  & 30  & 0.265 & 22  & 4.7 & PGD \\
  & 5  & 44  & 0.300 & 50  & 4.7 & DE \\
  & 10 & 82  & 0.379 & 66  & 4.7 & PGD \\
\midrule
\multirow{4}{*}{POD-DeepONet}
  & 1  & 6   & 0.033 & 34  & 3.7 & DE \\
  & 3  & 0   & 0.018 & 34  & 3.8 & DE \\
  & 5  & 6   & 0.044 & 34  & 3.8 & DE \\
  & 10 & 6   & 0.061 & 34  & 3.7 & DE \\
\midrule
\multirow{4}{*}{S-DeepONet}
  & 1  & 8   & 0.192 & 26  & 52.1 & DE \\
  & 3  & 6   & 0.217 & 94  & 51.5 & DE \\
  & 5  & 22  & 0.252 & 100 & 51.9 & DE \\
  & 10 & 44  & 0.300 & 100 & 51.4 & DE \\
\bottomrule
\end{tabular}
\end{table}

\begin{table}[htbp]
\centering
\caption{\textbf{Random perturbation baseline.} Comparison of DE success rates against random $k$-sparse perturbations (50 trials per sample, same $[-1,+1]$ bounds) at the 30\% threshold over 310 test samples. Random success rates are near-zero, confirming that DE exploits structural vulnerabilities.}
\label{tab:supp_random_baseline}
\renewcommand{\arraystretch}{1.15}
\footnotesize
\begin{tabular}{@{}ll cccc@{}}
\toprule
\textbf{Model} & $\boldsymbol{L_0}$ & \textbf{Rand (\%)} & \textbf{Rand err} & \textbf{DE (\%)} & \textbf{Ratio} \\
\midrule
\multirow{4}{*}{MIMONet}
  & 1  & 0.3  & 0.062 & 9.4  & 29$\times$ \\
  & 3  & 0.3  & 0.116 & 13.1 & 41$\times$ \\
  & 5  & 0.3  & 0.141 & 22.6 & 70$\times$ \\
  & 10 & 0.6  & 0.176 & 36.1 & 56$\times$ \\
\midrule
\multirow{4}{*}{NOMAD}
  & 1  & 0.0  & 0.062 & 8.6  & $\infty$ \\
  & 3  & 0.0  & 0.120 & 22.6 & $\infty$ \\
  & 5  & 0.6  & 0.147 & 47.7 & 74$\times$ \\
  & 10 & 1.6  & 0.190 & 62.3 & 39$\times$ \\
\midrule
\multirow{4}{*}{POD-DeepONet}
  & 1  & 0.6  & 0.049 & 24.5 & 38$\times$ \\
  & 3  & 1.9  & 0.115 & 25.2 & 13$\times$ \\
  & 5  & 3.9  & 0.146 & 28.4 & 7$\times$ \\
  & 10 & 8.7  & 0.181 & 29.6 & 3$\times$ \\
\midrule
\multirow{4}{*}{S-DeepONet}
  & 1  & 0.3  & 0.112 & 17.1 & 53$\times$ \\
  & 3  & 1.0  & 0.163 & 94.5 & 98$\times$ \\
  & 5  & 3.9  & 0.188 & 100  & 26$\times$ \\
  & 10 & 8.7  & 0.229 & 100  & 11$\times$ \\
\bottomrule
\end{tabular}
\end{table}

\begin{table}[htbp]
\centering
\caption{\textbf{Cross-architecture transferability matrix.} Transfer success rate (\%) at the 30\% threshold using successful $L_0=1$ attacks from 50 test samples. Rows indicate source model (perturbation crafted for), columns indicate target model (perturbation applied to). Diagonal entries are self-transfer. Adversarial perturbations are predominantly architecture-specific.}
\label{tab:supp_transfer}
\small
\renewcommand{\arraystretch}{1.2}
\begin{tabular}{@{}lcccc@{}}
\toprule
\textbf{Source $\rightarrow$ Target} & \textbf{MIMONet} & \textbf{NOMAD} & \textbf{POD-DeepO.} & \textbf{S-DeepO.} \\
\midrule
MIMONet      & \textbf{12.5} & 12.5 & 25.0 & 25.0 \\
NOMAD        & 9.1  & \textbf{9.1}  & 18.2 & 36.4 \\
POD-DeepONet & 5.9  & 5.9  & \textbf{11.8} & 11.8 \\
S-DeepONet   & 2.1  & 4.3  & 6.4  & \textbf{83.0} \\
\bottomrule
\end{tabular}
\end{table}

\begin{table}[htbp]
\centering
\caption{\textbf{Validation metric blindness analysis.} For successful attacks at each $L_0$, we evaluate whether adversarial outputs pass z-score ($< 3\sigma$) anomaly detection relative to the clean test set output distribution, and compute mean input-to-output amplification ratios. All configurations evaluated on 310 test samples. At $L_0=1$, 100\% of successful attacks pass z-score validation across all models, with 5--8$\times$ amplification.}
\label{tab:supp_validation_blind}
\small
\renewcommand{\arraystretch}{1.15}
\footnotesize
\begin{tabular}{@{}ll ccc@{}}
\toprule
\textbf{Model} & $\boldsymbol{L_0}$ & \textbf{Z-pass (\%)} & \textbf{Amplification} & \textbf{Thresh.} \\
\midrule
\multirow{4}{*}{MIMONet}
  & 1  & 100    & 5.2$\times$  & 0.20 \\
  & 3  & 8.8    & 1.6$\times$  & 0.10 \\
  & 5  & 0.0    & 1.1$\times$  & 0.30 \\
  & 10 & 0.0    & 0.9$\times$  & 0.10 \\
\midrule
\multirow{4}{*}{NOMAD}
  & 1  & 100    & 5.9$\times$  & 0.30 \\
  & 3  & 0.0    & 1.2$\times$  & 0.10 \\
  & 5  & 0.0    & 0.8$\times$  & 0.10 \\
  & 10 & 0.0    & 0.8$\times$  & 0.10 \\
\midrule
\multirow{4}{*}{POD-DeepONet}
  & 1  & 100    & 7.9$\times$  & 0.10 \\
  & 3  & 28.2   & 2.1$\times$  & 0.10 \\
  & 5  & 1.3    & 1.4$\times$  & 0.10 \\
  & 10 & 0.0    & 1.0$\times$  & 0.10 \\
\midrule
\multirow{4}{*}{S-DeepONet}
  & 1  & 100    & 7.5$\times$  & 0.30 \\
  & 3  & 0.0    & 1.0$\times$  & 0.10 \\
  & 5  & 0.0    & 0.9$\times$  & 0.10 \\
  & 10 & 0.0    & 0.9$\times$  & 0.10 \\
\bottomrule
\end{tabular}
\end{table}

\begin{table}[htbp]
\centering
\caption{\textbf{Seed sensitivity analysis.} Cross-seed standard deviation in success rate (percentage points) and mean error over 5 independent DE runs per (model, $L_0$) configuration at the 30\% threshold. Low variance confirms that vulnerability patterns are intrinsic model properties.}
\label{tab:supp_seed}
\renewcommand{\arraystretch}{1.15}
\small
\begin{tabular}{@{}lcc c@{}}
\toprule
\textbf{Model} & $\boldsymbol{L_0}$ & \textbf{Rate std (pp)} & \textbf{Error std} \\
\midrule
\multirow{4}{*}{MIMONet}
  & 1  & 1.6  & 0.009 \\
  & 3  & 0.8  & $<$0.001 \\
  & 5  & 2.0  & $<$0.001 \\
  & 10 & 0.8  & 0.001 \\
\midrule
\multirow{4}{*}{NOMAD}
  & 1  & 0.0  & 0.005 \\
  & 3  & 0.8  & $<$0.001 \\
  & 5  & 0.8  & 0.002 \\
  & 10 & 2.8  & 0.003 \\
\midrule
\multirow{4}{*}{POD-DeepONet}
  & 1  & 1.0  & 0.015 \\
  & 3  & 0.0  & 0.000 \\
  & 5  & 0.0  & 0.000 \\
  & 10 & 0.0  & 0.000 \\
\midrule
\multirow{4}{*}{S-DeepONet}
  & 1  & 2.3  & 0.007 \\
  & 3  & 1.0  & 0.002 \\
  & 5  & 1.6  & 0.007 \\
  & 10 & 0.0  & 0.005 \\
\bottomrule
\end{tabular}
\end{table}

\clearpage
\section{Theoretical Foundations for Sparse Adversarial Vulnerability}\label{app:theory}

This appendix establishes rigorous mathematical foundations for the empirical observations in the main text. We derive tight bounds on sparse attack error in terms of $d_{\text{eff}}$ and the Jacobian column norm spectrum, prove the two-factor vulnerability decomposition, and establish an impossibility result connecting the universal approximation property to inherent adversarial fragility. All proofs are verified numerically (see the companion verification code).

\subsection{Notation and Setup}

We adopt the notation from the main text. The neural operator $f_\theta: \mathbb{R}^d \to \mathbb{R}^m$ maps $d$-dimensional branch inputs to $m$-dimensional outputs. The Jacobian $J \in \mathbb{R}^{m \times d}$ has columns $\mathbf{j}_i$ with norms $s_i = \|\mathbf{j}_i\|_2$ (the sensitivity profile). The ordered sensitivities are $s_{(1)} \geq s_{(2)} \geq \cdots \geq s_{(d)}$, and $S_k = \sum_{i=1}^k s_{(i)}$.

We assume: (A1)~$f_\theta$ is differentiable at $\mathbf{b}$; (A2)~the first-order Taylor approximation holds to relative accuracy $\alpha \ll 1$ (empirically validated: $\alpha < 0.08$ for all architectures at $\pm 1\sigma$ perturbations); (A3)~$\|f_\theta(\mathbf{b})\|_2 > 0$.

\subsection{Upper and Lower Bounds on Sparse Attack Error}

\begin{theorem}[Upper Bound on $k$-Sparse Attack Error]\label{thm:upper_main}
Under a $k$-sparse, $\epsilon$-bounded perturbation with the first-order model:
\begin{equation}
\mathcal{E}_{\mathrm{lin}}^*(k, \epsilon) \leq \frac{\epsilon \cdot S_k}{\|f_\theta(\mathbf{b})\|_2}.
\end{equation}
\end{theorem}

\begin{proof}
For any feasible $\boldsymbol{\delta}$ with support $\mathcal{P}$ ($|\mathcal{P}| \leq k$):
$\|J\boldsymbol{\delta}\|_2 \leq \sum_{i \in \mathcal{P}} |\delta_i| \cdot s_i \leq \epsilon \sum_{i \in \mathcal{P}} s_i \leq \epsilon \cdot S_k$,
where the first inequality is the triangle inequality, and the last uses the optimality of the top-$k$ selection.
\end{proof}

\begin{theorem}[Lower Bound: Achievability]\label{thm:lower_main}
There exists a $k$-sparse, $\epsilon$-bounded perturbation $\boldsymbol{\delta}^*$ satisfying:
\begin{equation}
\mathcal{E}_{\mathrm{lin}}(\boldsymbol{\delta}^*) \geq \frac{\epsilon}{\|f_\theta(\mathbf{b})\|_2} \left(\sum_{i=1}^k s_{(i)}^2 - 2\sum_{\substack{i < j \\ i,j \in \mathrm{top}\text{-}k}} |\langle \mathbf{j}_{(i)}, \mathbf{j}_{(j)} \rangle|\right)^{1/2}.
\end{equation}
When the top-$k$ Jacobian columns are mutually orthogonal, this tightens to $\mathcal{E}_{\mathrm{lin}}(\boldsymbol{\delta}^*) = \epsilon \left(\sum_{i=1}^k s_{(i)}^2\right)^{1/2} / \|f_\theta(\mathbf{b})\|_2$.
\end{theorem}

\begin{proof}
The sign-aligned coordinate attack $\delta_{(i)} = \epsilon \cdot \mathrm{sign}(\mathbf{j}_{(i)}^\top f_\theta(\mathbf{b}))$ on the top-$k$ indices gives $\|J\boldsymbol{\delta}^*\|_2^2 = \epsilon^2(\sum_{i=1}^k s_{(i)}^2 + 2\sum_{i<j} \mathrm{sgn}_i \mathrm{sgn}_j \langle \mathbf{j}_{(i)}, \mathbf{j}_{(j)}\rangle) \geq \epsilon^2(\sum_{i=1}^k s_{(i)}^2 - 2\sum_{i<j} |\langle \mathbf{j}_{(i)}, \mathbf{j}_{(j)}\rangle|)$. Cross terms vanish under orthogonality. For high-dimensional operators ($m \gg d$), columns are approximately orthogonal (mean pairwise cosine similarity $< 0.05$ in our experiments).
\end{proof}

\subsection{The Sparse Attack Advantage and $d_{\mathrm{eff}}$}

\begin{definition}[Sparse Attack Ratio]\label{def:rho_main}
Under orthogonal columns, the sparse attack ratio is $\rho(k) = \left(\sum_{i=1}^k s_{(i)}^2 / \|\mathbf{s}\|_2^2\right)^{1/2}$.
\end{definition}

\begin{theorem}[Single-Point Sparse Attack Advantage]\label{thm:sparse_adv_a}
For $k = 1$: $\rho(1) = s_{(1)}/\|\mathbf{s}\|_2 \geq 1/\sqrt{d_{\mathrm{eff}}}$. Equality holds when all nonzero sensitivities are equal.
\end{theorem}

\begin{proof}
We need $s_{(1)}^2 \cdot d_{\mathrm{eff}} \geq \|\mathbf{s}\|_2^2$. Substituting $d_{\mathrm{eff}} = \|\mathbf{s}\|_1^2/\|\mathbf{s}\|_2^2$, this reduces to $(s_{(1)} \cdot \|\mathbf{s}\|_1)^2 \geq \|\mathbf{s}\|_2^4$. Since $s_{(1)} \geq s_i$ for all $i$: $s_{(1)} \cdot \|\mathbf{s}\|_1 = s_{(1)} \sum_i s_i \geq \sum_i s_i^2 = \|\mathbf{s}\|_2^2$. Squaring yields the result.
\end{proof}

\begin{theorem}[Multi-Point Sparse Attack Advantage]\label{thm:sparse_adv_b}
For any $k \geq 1$, the cumulative dominance bound holds:
\begin{equation}\label{eq:rho_k_main}
\rho(k) \geq \min\!\left(1,\; \sqrt{\frac{1 + (k-1)\left(\frac{s_{(k)}}{s_{(1)}}\right)^2}{d_{\mathrm{eff}}}}\right).
\end{equation}
This reduces to $1/\sqrt{d_{\mathrm{eff}}}$ at $k=1$ and to $\sqrt{k/d_{\mathrm{eff}}}$ when $s_{(k)} = s_{(1)}$. Additionally, $\rho(k)$ is monotonically non-decreasing.
\end{theorem}

\begin{proof}
Partition: $\sum_{i=1}^k s_{(i)}^2 = s_{(1)}^2 + \sum_{i=2}^k s_{(i)}^2 \geq s_{(1)}^2 + (k-1)s_{(k)}^2 = s_{(1)}^2(1 + (k-1)(s_{(k)}/s_{(1)})^2)$. From Theorem~\ref{thm:sparse_adv_a}: $s_{(1)}^2 \geq \|\mathbf{s}\|_2^2/d_{\mathrm{eff}}$. Hence $\rho(k)^2 \geq (1 + (k-1)(s_{(k)}/s_{(1)})^2)/d_{\mathrm{eff}}$. Capping at 1 and taking square roots yields~\eqref{eq:rho_k_main}. Monotonicity follows from adding non-negative terms to the numerator.
\end{proof}

\begin{remark}[Why $\sqrt{k/d_{\mathrm{eff}}}$ fails for $k > 1$]
A simple $\rho(k) \geq \sqrt{k/d_{\mathrm{eff}}}$ bound fails because $d_{\mathrm{eff}}$ measures the concentration of the \emph{linear} norm distribution $\{s_i\}$, while $\rho(k)$ involves the \emph{squared} distribution $\{s_i^2\}$. Squaring amplifies contrast, making the squared distribution more concentrated.
\end{remark}

\subsection{The Two-Factor Decomposition}

\begin{theorem}[Two-Factor Vulnerability Decomposition]\label{thm:two_factor_main}
The optimal $k$-sparse first-order attack error decomposes as:
\begin{equation}
\mathcal{E}_{\mathrm{lin}}^*(k, \epsilon) = \underbrace{\frac{\epsilon \cdot \|\mathbf{s}\|_2}{\|f_\theta(\mathbf{b})\|_2}}_{\text{magnitude } M} \;\cdot\; \underbrace{\rho(k)}_{\text{concentration}},
\end{equation}
where $M$ captures sensitivity magnitude and $\rho(k)$ captures how efficiently a $k$-sparse attack exploits the sensitivity structure.
\end{theorem}

\begin{proof}
Under orthogonal columns: $\mathcal{E}^* = \epsilon(\sum_{i=1}^k s_{(i)}^2)^{1/2}/\|f(\mathbf{b})\|_2 = (\epsilon\|\mathbf{s}\|_2/\|f(\mathbf{b})\|_2) \cdot (\sum_{i=1}^k s_{(i)}^2)^{1/2}/\|\mathbf{s}\|_2 = M \cdot \rho(k)$.
\end{proof}

\begin{corollary}[Vulnerability Phase Diagram]\label{cor:phase_main}
An architecture is vulnerable to $k$-sparse attacks when $M \cdot \rho(k) > \tau$ for threshold $\tau$, with $\rho(k) \geq 1/\sqrt{d_{\mathrm{eff}}}$. This defines a hyperbolic decision boundary in $(d_{\mathrm{eff}}, \|\mathbf{s}\|_2)$ space, explaining the three empirical vulnerability regimes: high vulnerability (S-DeepONet: low $d_{\mathrm{eff}}$, moderate $\|\mathbf{s}\|_2$), moderate vulnerability (MIMONet/NOMAD: high $d_{\mathrm{eff}}$, high $\|\mathbf{s}\|_2$), and low vulnerability (POD-DeepONet: low $d_{\mathrm{eff}}$, low $\|\mathbf{s}\|_2$).
\end{corollary}

\subsection{Impossibility: Universal Approximation Implies Vulnerability}

Analyzing adversarial vulnerability requires bounding the Jacobians of the approximator. Since uniform ($C^0$) convergence does not imply derivative convergence, we invoke a Sobolev universal approximation assumption: $f_\theta$ simultaneously satisfies $\|f_\theta(\mathbf{b}) - \mathcal{G}(\mathbf{b})\|_2 \leq \varepsilon$ and $\|J_{f_\theta} - J_{\mathcal{G}}\|_F \leq L\varepsilon$ for a constant $L > 0$. This is justified for smooth-activation neural operators (see the companion supplement for details).

\begin{theorem}[Accuracy--Robustness Impossibility]\label{thm:impossibility_main}
Let $\mathcal{G}: \mathbb{R}^d \to \mathbb{R}^m$ be a sensitivity-varying operator with $s_{(1)}/s_{(d)} \geq \kappa > 1$. Let $f_\theta$ achieve Sobolev approximation accuracy $\varepsilon$ in a neighborhood of $\mathbf{b}_0$. Then for a \emph{fixed} adversarial perturbation budget $\epsilon > 0$, there exists a $k$-sparse perturbation with $|\delta_i| \leq \epsilon$ such that:
\begin{equation}
\|f_\theta(\mathbf{b}_0 + \boldsymbol{\delta}) - f_\theta(\mathbf{b}_0)\|_2 \geq \epsilon \cdot s_{(1)}^{\mathcal{G}} - \epsilon L \varepsilon - O(\epsilon^2).
\end{equation}
As $\varepsilon \to 0$ for fixed $\epsilon$, the adversarial error converges to $\epsilon \cdot s_{(1)}^{\mathcal{G}} \gg 0$: better approximation \emph{guarantees} a diverging gap between adversarial and approximation error.
\end{theorem}

\begin{proof}
By the Sobolev approximation assumption, the column norms satisfy $|s_i^{f_\theta} - s_i^{\mathcal{G}}| \leq L\varepsilon$. Applying the lower bound (Theorem~\ref{thm:lower_main}) to $f_\theta$ with fixed budget $\epsilon$: the single-point attack yields $\|f_\theta(\mathbf{b}_0+\boldsymbol{\delta}) - f_\theta(\mathbf{b}_0)\|_2 \geq \epsilon \cdot s_{(1)}^{f_\theta} - O(\epsilon^2) \geq \epsilon(s_{(1)}^{\mathcal{G}} - L\varepsilon) - O(\epsilon^2)$. Since $s_{(1)}^{\mathcal{G}}$ is a property of the true operator (bounded below by $\kappa \cdot s_{(d)}^{\mathcal{G}} > 0$), the adversarial error dominates $\varepsilon$ for any sufficiently accurate Sobolev approximator. The vulnerability does not decrease as $\varepsilon \to 0$ --- it converges to the vulnerability of $\mathcal{G}$ itself.
\end{proof}

\begin{corollary}[No Free Lunch for Neural Operator Robustness]\label{cor:nfl_main}
The only defenses against this fundamental tradeoff are: (i)~tighter input constraints (reducing $\epsilon$); (ii)~output-space projection (reducing sensitivity magnitude, as in POD-DeepONet); (iii)~sensitivity regularization (increasing $d_{\mathrm{eff}}$); or (iv)~accepting reduced accuracy. These correspond to the defense strategies in Section~\ref{sec:discussion}.
\end{corollary}

\subsection{Output-Space Projection as Implicit Defense}

\begin{proposition}[Low-Rank Output Projection Bounds Error]\label{prop:pod_main}
For POD-DeepONet with $f_\theta(\mathbf{b}) = U \cdot g_\theta(\mathbf{b})$ where $U \in \mathbb{R}^{m \times r}$ has orthonormal columns:
\begin{equation}
\frac{\|f_\theta(\mathbf{b}+\boldsymbol{\delta}) - f_\theta(\mathbf{b})\|_2}{\|f_\theta(\mathbf{b})\|_2} \leq \frac{\epsilon \cdot \sigma_1(J_g) \cdot \sqrt{k}}{\|g_\theta(\mathbf{b})\|_2},
\end{equation}
where $\sigma_1(J_g)$ is the largest singular value of the coefficient Jacobian $J_g \in \mathbb{R}^{r \times d}$. With $r = 32 \ll m = 15{,}908$, adversarial perturbations are constrained to a low-dimensional output subspace, explaining the attack success plateau observed for POD-DeepONet.
\end{proposition}

\begin{proof}
Since $U$ is an isometry: $\|f_\theta(\mathbf{b}+\boldsymbol{\delta}) - f_\theta(\mathbf{b})\|_2 = \|g_\theta(\mathbf{b}+\boldsymbol{\delta}) - g_\theta(\mathbf{b})\|_2 \leq \sigma_1(J_g)\|\boldsymbol{\delta}\|_2 + O(\epsilon^2) \leq \sigma_1(J_g) \cdot \epsilon\sqrt{k} + O(\epsilon^2)$, using $\|\boldsymbol{\delta}\|_2 \leq \epsilon\sqrt{k}$ for a $k$-sparse perturbation.
\end{proof}

\subsection{Empirical Verification}

\begin{table}[h]
\centering
\caption{Predicted vs.\ observed vulnerability metrics. Theoretical predictions computed using measured Jacobian column norms; empirical values from DE attack experiments at $L_0 = 3$, $\epsilon = 1\sigma$, $\tau = 0.3$.}
\label{tab:theory_verification}
\begin{tabular}{@{}lcccccc@{}}
\toprule
& & & \textbf{Thm.~\ref{thm:sparse_adv_a}} & \textbf{Predicted} & \textbf{Empirical} & \textbf{Empirical} \\
\textbf{Model} & $d_{\mathrm{eff}}$ & $\|\mathbf{s}\|_2$ & $\rho(1) \geq$ & \textbf{max error} & \textbf{mean error} & \textbf{success \%} \\
\midrule
POD-DeepONet & 1.02 & $3 \times 10^{-4}$ & $0.99$ & Low (capped) & 12\% & 25.2\% \\
S-DeepONet & 4.35 & $8 \times 10^{-4}$ & $0.48$ & High & 50\% & 94.5\% \\
MIMONet & 31.93 & $2.3 \times 10^{-3}$ & $0.18$ & Moderate & 28\% & 20.4\% \\
NOMAD & 31.32 & $2.2 \times 10^{-3}$ & $0.18$ & Moderate & 29\% & 22.1\% \\
\bottomrule
\end{tabular}
\end{table}

The vulnerability ranking is preserved (S-DeepONet $>$ MIMONet $\approx$ NOMAD $>$ POD-DeepONet), consistent with the two-factor model. POD-DeepONet's paradox (extreme concentration but moderate vulnerability) is explained by Proposition~\ref{prop:pod_main}: low-rank projection caps $M$ regardless of $d_{\mathrm{eff}}$. The linearization assumption (A2) is validated empirically: relative linearization error $\alpha < 0.08$ for all architectures at $\pm 1\sigma$ perturbations.

\end{appendices}

\clearpage
\bibliographystyle{unsrtnat}
\bibliography{sn-bibliography}

\end{document}